\documentclass{article} %
\usepackage{preprint,times}

\usepackage{amsmath,amsfonts,bm}

\def\1{\bm{1}}

\DeclareMathAlphabet{\mathsfit}{\encodingdefault}{\sfdefault}{m}{sl}
\SetMathAlphabet{\mathsfit}{bold}{\encodingdefault}{\sfdefault}{bx}{n}

\newcommand{\E}{\mathbb{E}}

\newcommand{\R}{\mathbb{R}}

\usepackage{amsmath,amssymb,amsthm,mathtools}
\usepackage{booktabs}
\usepackage{graphicx}
\usepackage{enumitem}
\usepackage{microtype}
\newtheorem{theorem}{Theorem}
\newtheorem{proposition}{Proposition}
\newtheorem{lemma}{Lemma}
\newtheorem{corollary}{Corollary}
\usepackage{hyperref}
\hypersetup{hidelinks,pdftitle={Anisotropic Representations Improve Planning in JEPA World Models},pdfauthor={Mingu Kang, Yoori Oh, Sookyung Kim, Joonseok Lee}}
\usepackage{url}
\usepackage[most]{tcolorbox}
\usepackage{xcolor}
\usepackage{titletoc}

\definecolor{linkblue}{rgb}{0.21,0.49,0.74}

\hypersetup{
    colorlinks=true,
    urlcolor=linkblue,
    linkcolor=black,
    citecolor=black
}

\usepackage[capitalize]{cleveref}
\crefname{section}{Sec.}{Secs.}
\Crefname{section}{Section}{Sections}
\crefname{table}{Tab.}{Tabs.}
\Crefname{table}{Table}{Tables}
\crefname{figure}{Fig.}{Figs.}
\Crefname{figure}{Figure}{Figures}
\crefname{appendix}{App.}{Apps.}
\Crefname{appendix}{Appendix}{Appendices}

\input{function}

\title{Anisotropic Representations \\Improve Planning in JEPA World Models}

\author{
Mingu Kang$^{1}$ \quad
Yoori Oh$^{1}$\thanks{Corresponding authors} \quad
Sookyung Kim$^{2}$\footnotemark[1] \quad
Joonseok Lee$^{1}$\footnotemark[1] \\
$^{1}$Seoul National University \\
$^{2}$Ewha Womans University \\
\footnotesize
\texttt{\{gms5560, yoori0203, joonseok\}@snu.ac.kr}\\
\texttt{sookim@ewha.ac.kr}
}

\begin{document}

\maketitle
\begin{abstract}
Latent world models learn action-conditioned dynamics in representation space and often score candidate actions by Euclidean distance to a goal representation.
Joint training typically regularizes the representation to prevent collapse, but the resulting representation geometry also determines how terminal errors are weighted during planning.
We show that accurate prediction and noncollapsed representations do not guarantee a task-aligned latent planning cost: isotropic Gaussian regularization can induce a geometry that ranks feasible outcomes differently from the task cost.
To address this mismatch, we introduce AnisoWM with $\Lambda$Reg, which replaces the fixed isotropic Gaussian target with a learnable diagonal covariance under fixed-trace and anisotropy constraints.
The prediction objective, predictor architecture, and Euclidean planner remain unchanged; the target is used only during training.
Our analysis characterizes the prediction-driven allocation of target variance, its dependence on the training distribution, and the conditions under which the induced metric reduces planning regret.
Across four visual control environments, AnisoWM improves planning success over LeWorldModel in all four.
Its latent planning cost also shows better agreement with task outcomes. \textit{Project website: \href{https://rkdrn79.github.io/AnisoWM-page/}{https://rkdrn79.github.io/AnisoWM-page/}}
\end{abstract}
\section{Introduction}
\label{sec:intro}
JEPA-based world models \citep{maes2026lewm, assran2025vjepa2,zhou2025dinowm} predict how an agent's state will evolve under candidate actions in a learned representation space.
In visual goal planning, the planner rolls out candidate actions in this space and scores the predicted outcome by its distance to the goal representation, often using Euclidean distance.
As a result, the encoder does more than provide features for prediction: it also determines the geometry of the planning cost, and hence how different terminal errors are weighted.

LeWorldModel (LeWM) \citep{maes2026lewm} jointly learns an encoder and an action-conditioned predictor from visual observations.
Given a current observation and a visual goal, the encoder maps them into latent representations, while the predictor rolls out candidate action sequences in latent space.
LeWM scores an action sequence $U$ by
\begin{equation}
 J_z(U)=\norm{\hat z_H(U)-z_g}^2,
 \label{eq:planningcost}
\end{equation}
where $H$ is the planning horizon, $\hat z_H(U)$ is the predicted terminal representation, and $z_g$ is the goal representation.
Since the planner minimizes $J_z(U)$, its ranking of candidate outcomes should agree with the task cost.

However, LeWM does not explicitly optimize the representation geometry for alignment with the task cost.
The objective combines prediction loss with Sketched Isotropic Gaussian Regularization (SIGReg), introduced in LeJEPA \citep{balestriero2025lejepa}, to prevent representation collapse.
SIGReg encourages the learned features to follow an isotropic Gaussian distribution.
The isotropic target is motivated by theoretical criteria for downstream representation quality under linear and nonlinear probing, rather than by whether the resulting Euclidean distances are suitable for planning.
This mismatch motivates our central question: does the latent geometry learned by jointly optimizing the predictor and SIGReg rank candidate outcomes in the same order as the task cost?

We answer this question by analyzing how joint prediction--SIGReg training determines the geometry used for planning (\cref{sec:analysis}).
In a linear Gaussian setting, as process noise vanishes, the joint objective selects an approximately whitened representation, so Euclidean latent distance induces inverse-state-covariance weighting in state space.
This places greater emphasis on low-variance directions and can change the ordering of feasible outcomes relative to the task cost, leaving positive planning regret even with exact conditional-mean prediction in representation space.
A nonlinear MLP toy experiment exhibits the same prediction--planning separation.
Together, these results motivate learning how variance is allocated across latent directions rather than fixing every direction to the same target variance.

We therefore introduce AnisoWM with $\Lambda$Reg (\cref{sec:method}), which learns a diagonal Gaussian target jointly with the encoder and predictor.
We keep the total target variance fixed and bound its anisotropy by a condition-number constraint $\kappa$, while the allocation across latent directions is learned through joint prediction--regularization training.
The prediction objective, predictor architecture, and Euclidean planner remain unchanged. Our analysis shows that the learned target can counteract the inverse-covariance weighting induced by isotropic regularization, while also showing that excessive anisotropy can increase planning regret.

We evaluate AnisoWM on four visual goal-planning environments.
AnisoWM improves planning success over LeWM in all four environments.
Its latent costs also show better agreement with task outcomes, providing an empirical counterpart to the ordering mismatch highlighted by our analysis.
Under the same anisotropy bound, training learns different target spectra across environments, and increasing the allowed anisotropy does not monotonically improve planning.
Together, these results suggest that the learned variance allocation plays an important role in the observed planning improvements.

Our main contributions can be summarized as follows:
\begin{itemize}[leftmargin=10pt,itemsep=0pt,topsep=0pt]

  \item We show that the expected prediction objective with isotropic SIGReg can select a planning-misaligned latent geometry despite non-collapse and exact conditional-mean prediction.

  \item We introduce AnisoWM with $\Lambda$Reg, which replaces the fixed isotropic Gaussian target with a constrained anisotropic target whose variance allocation is learned jointly with the encoder and predictor. This learned variance allocation reshapes the latent geometry, and we theoretically show that it can better align latent distances with task costs and thereby reduce planning regret.

  \item We empirically verify that AnisoWM improves LeWM in the success rate of planning across four visual control environments, as well as the agreement between latent costs and recorded task outcomes.

\end{itemize}

\section{Related Work}
\label{sec:related}

\textbf{Latent world models. }
Latent world models learn predictive dynamics in representation space for planning and control.
PlaNet \citep{hafner2019planet} learns latent dynamics directly from pixels, while TD-MPC \citep{hansen2022tdmpc} learns task-oriented latent dynamics for model-predictive control.
More recent visual world models predict over learned or pretrained visual representations, including DINO-WM \citep{zhou2025dinowm} and V-JEPA-based approaches \citep{assran2025vjepa2}.
LeWorldModel (LeWM) \citep{maes2026lewm} jointly trains an encoder and action-conditioned predictor with SIGReg and plans using Euclidean distance between predicted and goal representations.
Related work modifies this pipeline in different ways: Fast-LeWM \citep{gao2026fastlewm} changes the predictive structure to reduce rollout cost and accumulated error, while RC-aux \citep{li2026rcaux} augments training with multi-horizon prediction and budget-conditioned reachability supervision.

\textbf{Representation regularization. }
Self-supervised objectives commonly constrain feature statistics to prevent collapse and redundancy.
Barlow Twins \citep{zbontar2021barlow}, whitening-based methods \citep{ermolov2021whitening}, and VICReg \citep{bardes2022vicreg} impose second-order constraints on learned representations.
LeJEPA \citep{balestriero2025lejepa} derives an isotropic Gaussian target from probing-risk criteria and introduces SIGReg to encourage representations to match that target.
Alternative representation distributions include radial Gaussianization in Radial-VCReg \citep{kuang2026radial} and sparse nonnegative targets in Rectified LpJEPA \citep{kuang2026rectified} and LpWM \citep{kuang2026lpwm}.
HamJEPA \citep{alvarez2026hamjepa} studies anisotropic Gaussian geometry derived from a prescribed structured geometry, whereas TC-LeWM \citep{liu2026tclewm} changes which features are regularized by SIGReg through temporal centering.

\textbf{Geometry for planning. }
A broader line of work studies representations whose geometry reflects control-relevant state similarity.
Bisimulation-based methods \citep{ferns2004metrics,zhang2021dbc} and DeepMDP \citep{gelada2019deepmdp} connect latent distances and dynamics to behavioral equivalence.
More recent work focuses directly on latent planning.
Temporal Straightening \citep{wang2026temporal} reduces trajectory curvature for gradient-based planning, while SCALE \citep{hu2026scale} calibrates LeWM distances against a task-relevant state space.
TRM \citep{li2026trm} learns a horizon-aware trajectory-reachability metric that replaces or augments the terminal planning cost, and Decision-Metric Alignment \citep{wang2026decisionmetric} introduces latent--outcome ranking diagnostics together with action-conditioned objectives for improving planning geometry.
Related identifiability results \citep{klindt2026when} characterize conditions under which predictive learning recovers state up to transformations that preserve Euclidean geometry. AnisoWM instead addresses planning geometry through the Gaussian representation regularizer while retaining the Euclidean planner.
\section{Analysis: Isotropic Regularization and Planning Cost}
\label{sec:analysis}

Throughout this section, we use \emph{metric matrix} to denote a positive-definite matrix $M\succ0$ that weights the discrepancy between two physical states $x,x_g\in\mathbb{R}^D$ through the quadratic cost
\[
 (x-x_g)^\top M(x-x_g).
\]

Since planning depends only on cost rankings, $M$ and $cM$ for any $c>0$ are equivalent for our purposes.
To show that isotropic regularization can induce planning regret, we first identify the state-space metric induced by isotropic feature covariance, then show that the expected joint prediction--SIGReg objective selects this metric in arbitrary dimension.
Finally, we connect the selected metric to finite-horizon planning regret.
Proofs and the explicit finite-horizon construction are given in Appendices~\ref{app:proofs} and~\ref{ag:appendix}.

\subsection{The metric induced by isotropic covariance}
\label{sec:metric}

An affine encoder with linear part $A$ induces the latent Euclidean cost
\[
 \norm{A\Delta x}^2
 =\Delta x^\top M\Delta x,
 \qquad
 M=A^\top A,
 \qquad
 \Delta x=x-x_g.
\]

For invertible $A$, $M\succ0$; for singular $A$, the same expression defines a positive-semidefinite quadratic cost.
The following lemma identifies the metric under isotropic feature covariance.

\begin{lemma}[Metric induced by isotropic covariance]
\label{lem:metric}
Let $A$ be the linear part of a square invertible affine encoder and let $\Sigma\succ0$ denote the state covariance.
If the encoded covariance is isotropic,
\[
 A\Sigma A^\top=sI
\]
for some $s>0$, then
\begin{equation}
 M=A^\top A=s\Sigma^{-1},
 \qquad
 \norm{A(x-x_g)}^2
 =s(x-x_g)^\top\Sigma^{-1}(x-x_g).
 \label{eq:whitening}
\end{equation}
Consequently, this latent Euclidean cost agrees up to positive scale with a task cost $\Delta x^\top Q\Delta x$, where $Q\succ0$, for all residuals if and only if $Q$ is proportional to $\Sigma^{-1}$.
\end{lemma}

Thus isotropic feature covariance does not in general induce an isotropic metric in the original state coordinates.
Instead, it weights errors by inverse state covariance: directions with lower training variance receive greater weight.
If the task metric and the latent metric weight directions differently, they can prefer different actions when planning requires trading off errors across directions.

\subsection{The metric selected by joint prediction--SIGReg training}
\label{sec:joint_geometry}

The preceding lemma is a geometric statement about an isotropic representation.
We now ask whether the expected joint prediction--SIGReg objective actually selects this geometry.
Consider independent Gaussian training tuples $(x,a,x')$ of state, action, and next state,
\begin{equation}
 \begin{aligned}
 x'&=F_\eta x+Ga+\xi,\\
 x&\sim\mathcal N(0,\Sigma),\qquad
 a\sim\mathcal N(0,\Gamma),\qquad
 \xi\sim\mathcal N(0,\eta W),
 \end{aligned}
 \label{eq:general_training_model}
\end{equation}
where $x$, $a$, and $\xi$ are mutually independent, $\Sigma,W\succ0$, $\Gamma\succ0$, and $G\in\mathbb R^{D\times d_a}$.

We optimize square linear encoders $A\in\mathbb R^{D\times D}$, including singular encoders to avoid assuming noncollapse, together with a linear predictor $p$ taking $(Ax,a)$ as input.
Write the prediction loss $\Lpred$ and the state-whitened process-noise covariance $R$ as
\[
 \Lpred(A,p)=\tfrac12\E\norm{p(Ax,a)-Ax'}^2,
 \qquad
 R=\Sigma^{-1/2}W\Sigma^{-1/2},
\]
and consider the isotropic-target objective,
\begin{equation}
 \mathcal L_B(A,p)
 =\Lpred(A,p)+\lambda_B\Reg(A\Sigma A^\top),
 \qquad \lambda_B>0.
 \label{eq:general_baseline_objective}
\end{equation}
Here $\Reg$ is the expected finite-batch SIGReg statistic characterized in Appendix~\ref{app:finite}, and $\lambda_B$ balances this regularization against prediction.
Under the finite-statistic assumptions there, $\Reg$ is uniquely minimized at $qI_D$ for some $q>0$, so it favors a noncollapsed isotropic feature covariance.
The common scale $q$ does not affect planning-cost rankings.

\begin{proposition}[Metric selected by isotropic joint training]
\label{prop:isotropic_joint_metric}
Assume the finite-statistic conditions of Lemma~\ref{lem:finite} and fix $\lambda_B>0$.
For all sufficiently small $\eta>0$, Eq.~\eqref{eq:general_baseline_objective} attains a global minimum.
Every minimizing encoder is invertible, even though singular encoders are admissible, and every minimizing predictor recovers the exact encoded conditional mean,
\[
 p_A(z,a)=AF_\eta A^{-1}z+AGa.
\]
All global minimizers induce the same Gram matrix
\[
 K_{B,\eta}=(A\Sigma^{1/2})^\top(A\Sigma^{1/2}),
\]
and, uniformly over the global minimizers,
\begin{equation}
 K_{B,\eta}\longrightarrow qI_D,
 \qquad
 M_{B,\eta}:=A^\top A\longrightarrow q\Sigma^{-1}
 \qquad\text{as }\eta\downarrow0.
 \label{eq:general_baseline_metric}
\end{equation}
The attained prediction loss is
\[
 \Lpred(A,p)=\tfrac\eta2\tr(K_{B,\eta}R)\longrightarrow0.
\]
\end{proposition}

The mechanism follows by reducing the joint objective to
\[
K=(A\Sigma^{1/2})^\top(A\Sigma^{1/2}).
\]
For an invertible encoder, the exact conditional-mean predictor attains the irreducible prediction loss $\eta\tr(KR)/2$, while orthogonal invariance makes the expected SIGReg term depend only on $K$.
The reduced objective is therefore
\[
 \frac{\eta}{2}\tr(KR)+\lambda_B\Reg(K).
\]
Proposition~\ref{prop:isotropic_joint_metric} ensures that every global minimizer is invertible for sufficiently small $\eta$, so this reduction applies at the optima of interest.
Prediction favors shrinking the encoding along high-noise directions, but as $\eta\downarrow0$ the isotropic regularizer determines the leading-order geometry, forcing $K\to qI_D$ and hence $A^\top A\to q\Sigma^{-1}$.
Thus joint training selects the inverse-covariance geometry despite noncollapse, exact encoded prediction, and vanishing prediction loss.

\subsection{Finite-horizon planning separation}
\label{sec:finite_horizon_separation}

We next ask whether the metric mismatch can change the action sequence selected over a finite planning horizon.
For a deterministic sequence $U=(a_0,\ldots,a_{H-1})$, let $J_{*,\eta,H}(U)$ denote the expected squared Euclidean terminal task cost, and define
\[
 \Regret_{\eta,H}(U)
 =
 J_{*,\eta,H}(U)-\min_V J_{*,\eta,H}(V).
\]

Under the finite-statistic conditions of Lemma~\ref{lem:finite}, consider any state dimension $D\ge2$, any fixed finite horizon $H\ge1$, any action bound $\bar u>0$, and any nonscalar state covariance $\Sigma\succ0$.
Appendix~\ref{ag:finite-horizon} constructs a fully actuated linear Gaussian control family with stationary covariance $\Sigma$, isotropic process noise, bounded per-step planning actions, and a fixed goal.

\begin{theorem}[Finite-horizon planning separation]
\label{thm:finite_horizon_separation}
For any fixed $\lambda_B>0$ and sufficiently small $\eta>0$, every global optimum of the expected prediction--SIGReg objective has a noncollapsed encoder and exact encoded conditional-mean rollouts, with training prediction loss and fixed-$H$ encoded rollout error vanishing as $\eta\downarrow0$.
Nevertheless, every exact Euclidean latent-planning minimizer $U_{B,\eta}$ satisfies
\begin{equation}
 \Regret_{\eta,H}(U_{B,\eta})
 \longrightarrow \Delta_H>0.
 \label{eq:finite_horizon_positive_regret}
\end{equation}
\end{theorem}

As $\eta\downarrow0$, the reachable terminal means in this construction form a Euclidean ball.
The task cost selects the Euclidean projection of the goal onto this ball, whereas Proposition~\ref{prop:isotropic_joint_metric} makes the latent planner select the projection under $\Sigma^{-1}$.
For a goal outside this ball with $\Sigma^{-1}x_g$ not parallel to $x_g$, the two projections differ, yielding the positive regret in Eq.~\eqref{eq:finite_horizon_positive_regret} even as prediction and fixed-$H$ rollout error vanish.
The mismatch arises because an unreachable goal forces the planner to trade off terminal errors across directions, which the inverse-covariance metric weights differently from the Euclidean task cost.

\subsection{Illustrating the Prediction--Planning Gap}
\label{sec:toy}

\begin{figure}[h]
    \centering
    \includegraphics[width=0.85\linewidth]{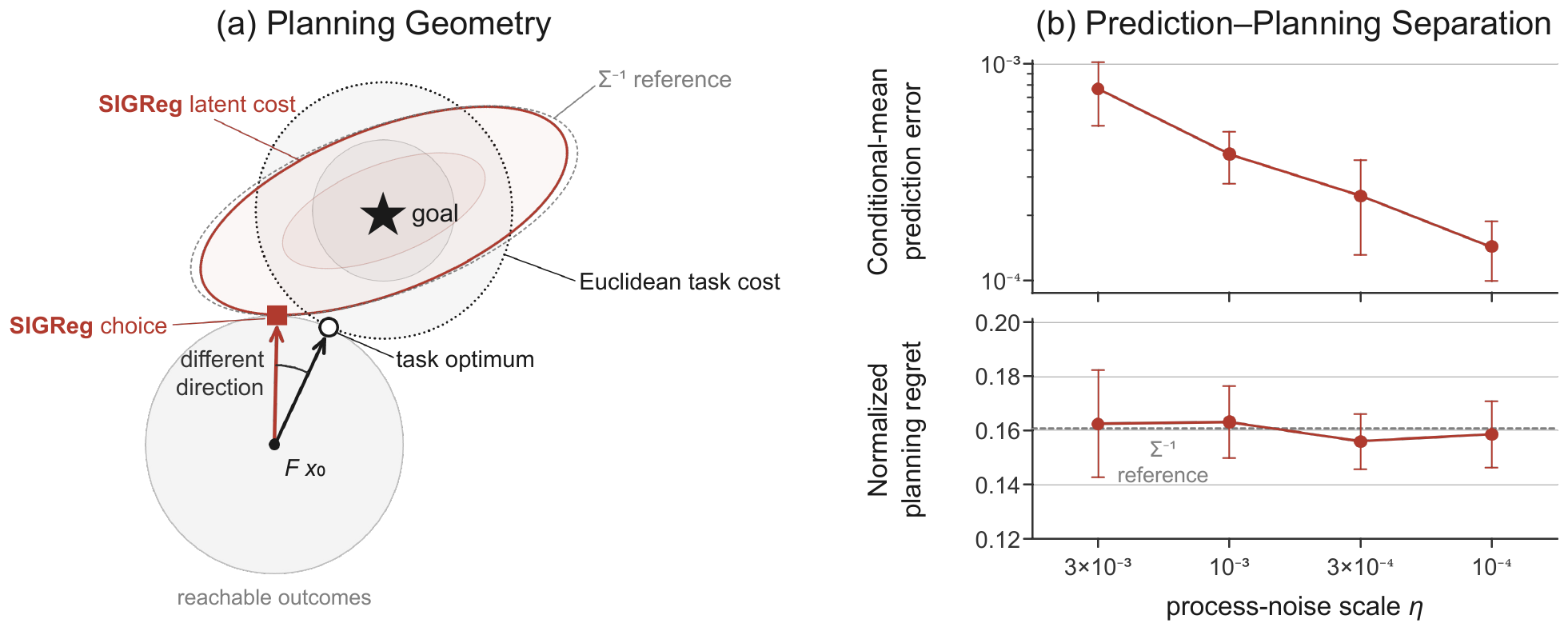}
    \caption{
    \textbf{Prediction--planning separation with a nonlinear encoder.}
    \textbf{(a)} The Euclidean task cost and the SIGReg latent cost select different outcomes from the same reachable set; the learned latent geometry closely follows the $\Sigma^{-1}$ reference.
    \textbf{(b)} As process noise decreases, held-out conditional-mean prediction error decreases while normalized physical planning regret remains nearly unchanged. Points and error bars show the mean and $95\%$ $t$-confidence interval over
ten training seeds.
    }
    \label{fig:toy_separation}
\end{figure}

To empirically illustrate Theorem~1, we test whether the prediction--planning separation persists beyond the linear setting on a two-dimensional controlled system designed to isolate the geometric mismatch analyzed above.
The training actions and process noise are isotropic, while the dynamics induce an anisotropic stationary state distribution.
At evaluation, goals are chosen outside the one-step reachable set, forcing the planner to trade off residual errors across state dimensions; the system is constructed so that the Euclidean task cost and the inverse-covariance latent cost prefer different reachable outcomes.
We train an MLP encoder and action-conditioned predictor with the prediction--SIGReg objective, using the nonlinear encoder to test whether the same behavior persists beyond the linear setting analyzed above.
As the process-noise scale decreases by $30\times$, held-out conditional-mean prediction error decreases substantially, while mean normalized physical planning regret remains near $0.16$ over ten seeds (Figure~\ref{fig:toy_separation}).
The observed regret closely matches the $0.1607$ inverse-covariance reference, and the learned local pullback metric is close to $\Sigma^{-1}$ up to scale.
Thus improved prediction does not remove the geometric action-ranking mismatch, motivating the anisotropic regularization introduced in \cref{sec:method}.

\section{Method}
\label{sec:method}

AnisoWM replaces SIGReg's fixed isotropic Gaussian target with a learnable diagonal covariance.
This permits different target variances across latent coordinates rather than imposing the same target variance in every direction.
The design follows the analysis in \cref{sec:analysis}: when Euclidean distance is used for planning, the regularization target also shapes how terminal errors are weighted by the learned representation.
We therefore relax the isotropy constraint during representation learning while leaving the prediction objective, predictor architecture, and Euclidean planner unchanged.

\subsection{Learnable Gaussian target}
\label{sec:learnable_target}

Let $f_\theta(o)\in\R^D$ be the encoder output and let
$\Lambda=\diag(v_1,\ldots,v_D)$ denote the covariance of a zero-mean Gaussian target.
We constrain
\begin{equation}
    \mathcal T_{D,\kappa}
    =
    \left\{
    \Lambda=\diag(v_1,\ldots,v_D)\succ0:
    \tr\Lambda=D,\quad
    \cond(\Lambda)\le\kappa
    \right\},
    \qquad
    \kappa\ge1,
    \label{eq:family}
\end{equation}
where $\cond(\Lambda)=\max_i v_i/\min_i v_i$.
The trace fixes the overall target scale, while $\kappa$ bounds the allowed anisotropy.
The target is diagonal in representation coordinates, while the encoder remains free to orient those coordinates relative to the underlying state geometry.
We initialize $\Lambda=I_D$; $\kappa=1$ recovers the isotropic target.

For a feature batch $Z=(z_1,\ldots,z_N)$, we define
\begin{equation}
    \widehat{\mathcal R}_{N,\Lambda}(Z;\boldsymbol\omega)
    =
    \widehat{\mathcal R}_N
    \!\left(
    \Lambda^{-1/2}Z;\boldsymbol\omega
    \right),
    \label{eq:targetregularizer}
\end{equation}
where $\widehat{\mathcal R}_N$ is the finite-batch SIGReg statistic
\citep{balestriero2025lejepa}.
Because $z\sim\mathcal N(0,\Lambda)$ implies
$\Lambda^{-1/2}z\sim\mathcal N(0,I_D)$, the original SIGReg reference distribution can be used unchanged.
This standardization is applied only inside the regularizer; prediction and planning use the original representation $z$.

\subsection{Joint training and planning}
\label{sec:joint_training}

The predictor produces
$\hat z_{t+1}=g_\phi(z_t,a_t)$ and is trained with mean squared prediction loss
$\Lpred(\theta,\phi;\mathcal B)$.
Let $Z_\theta$ denote the encoder features used by the regularizer.
We minimize
\begin{equation}
    \widehat{\mathcal L}_T
    =
    \Lpred(\theta,\phi;\mathcal B)
    +
    \lambda
    \widehat{\mathcal R}_{N,\Lambda}
    (Z_\theta;\boldsymbol\omega),
    \qquad
    \Lambda\in\mathcal T_{D,\kappa}.
    \label{eq:methodobjective}
\end{equation}
Both terms update the encoder, prediction loss updates the predictor, and
$\Lambda$ is updated only through the regularization term.

We parameterize
\begin{equation}
    v_i
    =
    D\frac{e^{\alpha_i}}{\sum_j e^{\alpha_j}},
    \qquad
    \alpha_i\in
    \left[-\tfrac12\log\kappa,\tfrac12\log\kappa\right],
    \label{eq:target_parameterization}
\end{equation}
with zero initialization and clipping after each update.
This enforces $\tr\Lambda=D$ and $\cond(\Lambda)\le\kappa$ and covers all of
$\mathcal T_{D,\kappa}$ (Appendix~\ref{app:target_parameterization}).

Planning uses autoregressive latent rollouts and is otherwise unchanged from LeWM:
\[
    J_z(U)=\norm{\hat z_H(U)-z_g}^2,
    \qquad
    z_g=f_\theta(o_g).
\]
The covariance $\Lambda$ is used only by the training regularizer and can be discarded afterward; thus AnisoWM changes the learned representation geometry without introducing an additional planning-time metric.

\subsection{Effect on planning geometry}
\label{sec:method_geometry}

We now characterize how the learnable target changes the state-space metric selected by joint training.
Return to the Gaussian model of Section~\ref{sec:joint_geometry}, with state covariance $\Sigma\succ0$ and state-whitened process-noise covariance
\[
    R=\Sigma^{-1/2}W\Sigma^{-1/2}.
\]
For a linear encoder $A$, write
\[
    K=(A\Sigma^{1/2})^\top(A\Sigma^{1/2}),
    \qquad
    A\Sigma^{1/2}=OK^{1/2},
    \qquad
    L=O^\top\Lambda O,
\]
where $O$ is orthogonal.
The matrix $L$ expresses the diagonal target covariance in state-whitened coordinates, including the orientation selected by the encoder.
Its feasible set is
\[
    \mathcal S_{D,\kappa}
    =
    \{L\succ0:\tr L=D,\ \cond(L)\le\kappa\}.
\]
Although $L$ need not be diagonal, it arises from the encoder orientation together with a diagonal $\Lambda$ and does not introduce a full-covariance target parameterization.

For the expected Gaussian model, the learned-target objective is
\begin{equation}
    \mathcal L_T(A,p,\Lambda)
    =
    \Lpred(A,p)
    +
    \lambda_T\Reg\!\left(
    \Lambda^{-1/2}A\Sigma A^\top\Lambda^{-1/2}
    \right),
    \qquad
    \Lambda\in\mathcal T_{D,\kappa}.
    \label{eq:general_target_objective}
\end{equation}

\begin{proposition}[Metric selected by target learning]
\label{prop:learned_target_metric}
Assume the finite-statistic conditions of Lemma~\ref{lem:finite}, and fix
$\lambda_T>0$ and $\kappa>1$.
For all sufficiently small $\eta>0$,
Eq.~\eqref{eq:general_target_objective} attains a global minimum.
Every global minimizer has an invertible encoder and an exact encoded conditional-mean predictor.
Uniformly over global minimizers,
\begin{equation}
    K-qL\longrightarrow0,
    \qquad
    A^\top A-q\Sigma^{-1/2}L\Sigma^{-1/2}\longrightarrow0.
    \label{eq:general_target_metric}
\end{equation}
Every accumulation point of $L$ minimizes
\begin{equation}
    \tr(LR)
    \qquad\text{over}\qquad
    L\in\mathcal S_{D,\kappa}.
    \label{eq:target_selection}
\end{equation}
If the minimizer is unique, then $L$ converges to it uniformly over global training optima, and the attained prediction loss tends to zero.
\end{proposition}

The proof is given in Appendix~\ref{ag:target-geometry}.

\textbf{Prediction-driven metric selection. }
Under the isotropic target, $L=I_D$, and Proposition~\ref{prop:isotropic_joint_metric} recovers the inverse-covariance metric $q\Sigma^{-1}$.
Learning the target enlarges the limiting family to
\[
    q\Sigma^{-1/2}L\Sigma^{-1/2},
    \qquad
    L\in\mathcal S_{D,\kappa}.
\]
Along global minimizers, $K=qL+o(1)$, so the prediction term is
\[
    \frac{\eta}{2}\tr(KR)
    =
    \frac{\eta q}{2}\tr(LR)+o(\eta),
\]
yielding the selection rule in Eq.~\eqref{eq:target_selection}.
Thus $\kappa$ bounds the admissible anisotropy, while predictive training selects how that anisotropy is allocated.

\textbf{Task alignment and finite-horizon compensation. }
For a task metric $Q\succ0$, the limiting representation metric is proportional to $Q$ when
\begin{equation}
    L_Q
    =
    \frac{D\Sigma^{1/2}Q\Sigma^{1/2}}
         {\tr(\Sigma Q)}.
    \label{eq:target_alignment}
\end{equation}
Exact alignment therefore requires both
$L_Q\in\mathcal S_{D,\kappa}$ and that predictive training select this geometry.
For the Euclidean task metric $Q=I_D$, $L_Q$ is the trace-normalized state covariance, showing how the learnable target can compensate for the inverse-covariance weighting induced by isotropic regularization without modifying the planner.

The finite-horizon construction of Theorem~\ref{thm:finite_horizon_separation} realizes this compensation explicitly.
For the trace-normalized two-level covariance family in Eq.~\eqref{ag:two-level}, the construction uses $W=I_D$, so $R=\Sigma^{-1}$.
With $\kappa=\cond(\Sigma)$, Eq.~\eqref{eq:target_selection} uniquely selects $L\to\Sigma$, and hence
\begin{equation}
    A^\top A\longrightarrow qI_D,
    \qquad
    \Regret_{\eta,H}(U_{T,\eta})\longrightarrow0.
    \label{eq:finite_horizon_compensation}
\end{equation}
The isotropic-target planner on the same family retains the positive limiting regret in Eq.~\eqref{eq:finite_horizon_positive_regret}, while both objectives have vanishing prediction loss and fixed-horizon encoded rollout error.

Appendix~\ref{app:targetfamily} gives closed-form regret analysis, while Appendix~\ref{app:adaptive_spectrum} analyzes higher-dimensional, distribution-dependent spectrum selection.

\section{Experiments}
\label{sec:experiments}
\subsection{Experimental setup}
\label{sec:experimental_setup}

We evaluate on TwoRoom, Reacher, PushT, and Cube using the datasets, model architecture, and visual goal-planning protocol of LeWM \citep{maes2026lewm}.
AnisoWM uses a $192$-dimensional representation and regularization weight $\lambda=0.09$. 
Implementation and regularizer hyperparameters are provided in Appendix~\ref{app:implementation}.
Each environment in the primary comparison is trained with three seeds, and the reported value is the mean over the three.
For the primary comparison, we use a single shared anisotropy bound $\kappa=2$ across all environments, which permits at most a twofold ratio between target variances while avoiding environment-specific tuning.
Prediction and CEM planning operate in the original latent coordinates.
Success rates for PLDM~\citep{sobal2025learning}, DINO-WM~\citep{zhou2025dinowm}, GCBC~\citep{ghosh2021gcsl}, GCIQL~\citep{kostrikov2022iql,park2025ogbench}, GCIVL~\citep{park2025ogbench}, and Random are taken from \citet{maes2026lewm}.
The LeWM baseline is trained in our pipeline with the released code and identical settings ($\lambda=0.09$, training seeds, and evaluation pairs); it corresponds to $\kappa=1$ in our parameterization.

\subsection{Results and Analysis}
\label{sec:planning_performance}

\textbf{Planning performance. }
AnisoWM plans more successfully than LeWM in all four environments (Figure~\ref{fig:main_result}): $93\%$ against $87\%$ in TwoRoom, $89\%$ against $86\%$ in Reacher, $97\%$ against $96\%$ in PushT, and $79\%$ against $74\%$ in Cube.
Because AnisoWM changes only the representation regularization, these gains preserve LeWM's predictor and Euclidean planner.
Although some reference baselines achieve higher success in TwoRoom and Cube, AnisoWM retains LeWM's efficient latent-space planning while achieving competitive performance.

\begin{figure}[t]
    \centering
    \includegraphics[width=\linewidth]{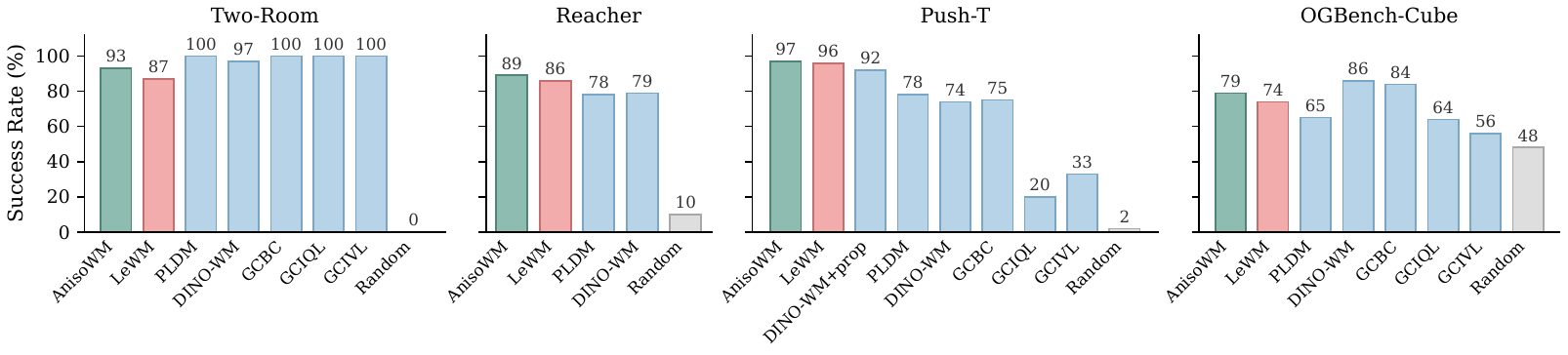}
    \vspace{-1cm}
    \caption{\textbf{Planning success across environments.}
    AnisoWM with $\Lambda$Reg at $\kappa=2$ in every environment, compared with LeWM and the reference baselines.
    Reported AnisoWM and LeWM values are means over three training seeds.}
    \label{fig:main_result}
\end{figure}

\begin{table}[t]
\vspace{-0.7cm}
\centering
\caption{\textbf{Ordering recorded action sequences by the outcome they reached.}
For each initial--goal pair, both models score the same sequences, and the entry is the fraction of sequence pairs whose cost ordering agrees with their outcome ordering.
$J_{\rm enc}$ encodes the observation a sequence actually reached; $J_{\rm pred}$ is the cost used for planning and additionally includes the predictor rollout.}
    \label{tab:action_ranking}
    \small
    \setlength{\tabcolsep}{7pt}
    \begin{tabular}{@{}lcccccc@{}}
        \toprule
        & \multicolumn{3}{c}{$J_{\rm enc}$ (representation only)}
        & \multicolumn{3}{c}{$J_{\rm pred}$ (planning cost)} \\
        \cmidrule(lr){2-4}\cmidrule(l){5-7}
        Task & LeWM & AnisoWM & Gap & LeWM & AnisoWM & Gap \\
        \midrule
        TwoRoom & 0.564 & 0.661 & $+0.097$ & 0.575 & 0.646 & $+0.071$ \\
        Reacher & 0.790 & 0.893 & $+0.103$ & 0.676 & 0.723 & $+0.048$ \\
        PushT   & 0.647 & 0.641 & $-0.005$ & 0.594 & 0.621 & $+0.028$ \\
        Cube    & 0.555 & 0.576 & $+0.021$ & 0.537 & 0.553 & $+0.016$ \\
        \bottomrule
    \end{tabular}
    \vspace{-0.5cm}
\end{table}

\textbf{Action ranking. }
\label{sec:action_selection}
Motivated by Theorem~\ref{thm:finite_horizon_separation}, we compare how LeWM and AnisoWM order the same recorded action sequences against their task outcomes; Appendix~\ref{app:ranking_protocol} gives the full protocol.
We report $J_{\rm pred}(u)=\|\hat z_H(u)-z_g\|^2$, the cost used by CEM, and $J_{\rm enc}(u)=\|f_\theta(o_H(u))-f_\theta(o_g)\|^2$, which evaluates the realized terminal observation without predictor rollout.
As shown in Table~\ref{tab:action_ranking}, AnisoWM improves $J_{\rm pred}$ ordering in all four environments.
Under $J_{\rm enc}$, the gain is positive in TwoRoom, Reacher, and Cube, while PushT is nearly unchanged, indicating that the contribution of representation geometry and predictor rollout differs across environments.
\vspace{-0.3cm}
\begin{figure}[t]
    \centering
    \includegraphics[width=0.8\textwidth]{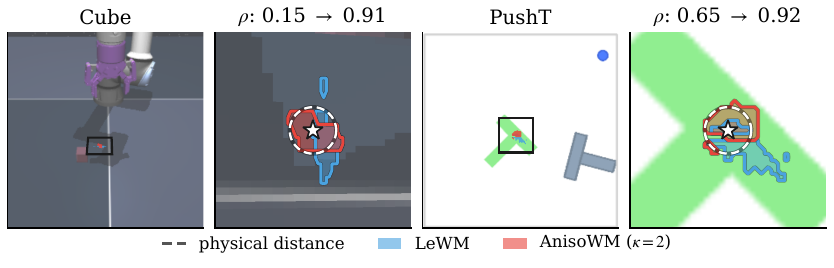}
    \vspace{-0.5cm}
    \caption{\textbf{Latent cost neighborhoods around the goal.}
    For Cube and PushT, the left panels show the rendered arena and evaluated region, and the right panels show positions in the lowest 10\% of latent planning cost around the goal ($\star$) for LeWM (blue) and AnisoWM ($\kappa{=}2$, red).
    The dashed circle shows the corresponding neighborhood under physical Euclidean distance.
    }
    \label{fig:cost-main}
\end{figure}

\paragraph{Local cost geometry.}
Figure~\ref{fig:cost-main} visualizes the latent planning cost around goals in Cube and PushT.
AnisoWM produces low-cost neighborhoods that more closely follow the corresponding task cost in physical distance,
 whereas LeWM assigns low cost to positions that can lie farther from the goal.
Spearman's rank correlation ($\rho$) between latent cost and task cost
increases from $0.15$ to $0.91$ in Cube and from $0.65$ to $0.92$ in PushT.
Additional examples are provided in Appendix~\ref{app:cost}.

\begin{figure}[t]
    \centering
\includegraphics[width=\linewidth]{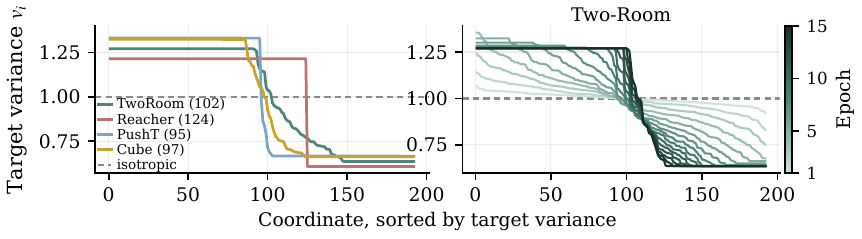}
\vspace{-0.7cm}
\caption{\textbf{Learned target spectra.}
Left: learned targets at $\kappa=2$, sorted by variance, with the isotropic target marked. The trace and maximum variance ratio are fixed, while the allocation is learned; the legend shows the number of coordinates above the mean variance.
Right: target spectrum during one TwoRoom run, showing continued evolution after the condition-number bound is reached.}
    \label{fig:target_spectra}
\end{figure}

\textbf{Learned target spectra.}
Despite the shared bound $\kappa=2$, the learned spectra differ across environments: the number of coordinates above the mean target variance ranges from $95$ in PushT to $124$ in Reacher (Figure~\ref{fig:target_spectra}).
In TwoRoom, the allocation continues to evolve after reaching the condition-number bound, showing that $\kappa$ constrains but does not determine the learned spectrum.

\begin{figure}[t]
    \centering
    \includegraphics[width=\linewidth]{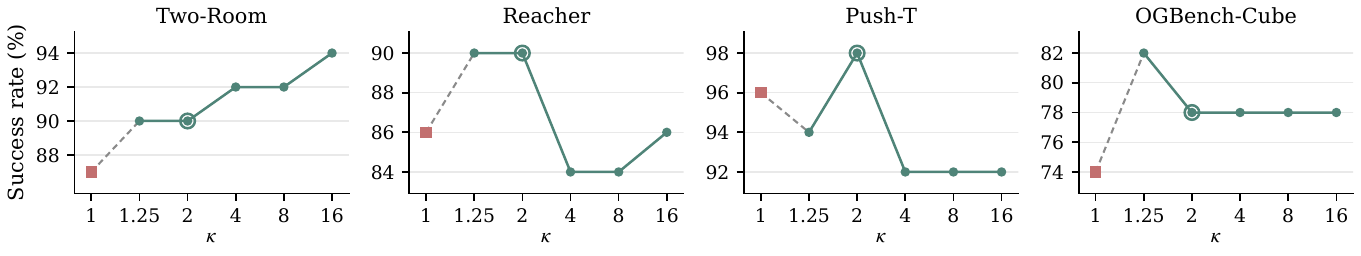}
    \vspace{-0.7cm}
    \caption{\textbf{Planning success across anisotropy bounds.}
    Each point with $\kappa>1$ reports one training run evaluated on the same set of initial--goal pairs.
    The $\kappa=1$ marker shows the isotropic LeWM baseline from the primary comparison.}
    \label{fig:kappa_sweep}
    \vspace{-0.5cm}
\end{figure}

\textbf{Sensitivity to the anisotropy bound.}
Figure~\ref{fig:kappa_sweep} shows non-monotone sensitivity to $\kappa$, with the highest observed anisotropic setting differing across environments.
Because each $\kappa>1$ point is a single training run, we treat the sweep as a sensitivity diagnostic rather than a tuned comparison. 
Appendix~\ref{app:sweep_details} reports the numerical values.

Additional experimental details are provided in Appendix~\ref{app:experiments}.

\section{Conclusion}

We studied how isotropic Gaussian regularization shapes the geometry used for latent planning.
Our analysis shows that accurate prediction and noncollapsed representations do not guarantee a task-aligned planning cost.
We introduced AnisoWM with $\Lambda$Reg, which learns a constrained anisotropic Gaussian target while leaving the prediction objective and Euclidean planner unchanged.
Our analysis characterizes the prediction-driven selection of representation geometry and establishes conditions under which it reduces planning regret.
Empirically, AnisoWM achieves higher planning success than LeWM across four visual control environments, and its latent costs show higher agreement with recorded task outcomes.
\clearpage

\bibliography{references}
\bibliographystyle{preprint}

\clearpage
\appendix
\clearpage
\hypersetup{
    colorlinks=true, %
    linkcolor=black, %
    filecolor=red, %
    urlcolor=red, %
}

\startcontents[appendix]
\section*{\textcolor{black}{Appendix}}
\printcontents[appendix]{}{1}{}\setcounter{tocdepth}{2}
\label{sec:appendix}

\section{Notation}
\begin{table}[h!]
\centering
\small
\renewcommand{\arraystretch}{1.08}
\setlength{\tabcolsep}{5pt}
\caption{Notation used throughout the paper.}
\label{tab:notation}
\begin{tabular}{p{0.27\linewidth} p{0.66\linewidth}}
\toprule
\textbf{Notation} & \textbf{Description} \\
\midrule
$o_t,\ o_g$ & Observation at time $t$ and goal observation. \\
$f_\theta,\ g_\phi$ & Neural encoder and action-conditioned latent predictor. \\
$z_t,\ z_g$ & Latent representation at time $t$ and goal representation. \\
$a_t,\ U$ & Per-step action and finite-horizon action sequence $U=(a_0,\ldots,a_{H-1})$. \\
$H$ & Planning horizon. \\
$\hat z_H(U)$ & Predicted terminal representation after rolling out $U$. \\
$J_z(U)$ & Euclidean latent terminal cost $\|\hat z_H(U)-z_g\|^2$. \\
$x,\ x'$ & Current and successor physical states in the theoretical analysis. \\
$x_g$ & Generic physical goal state. \\
$D,\ d_a$ & State/representation dimension and action dimension in the general theory. \\
$A$ & Square linear encoder in the theoretical analysis. \\
$p$ & Linear predictor in the theoretical analysis. \\
$M=A^\top A$ & State-space metric induced by Euclidean latent distance. \\
$Q$ & Positive-definite task metric. \\
$\Sigma$ & Training-state covariance. \\
$F_\eta,\ G,\ \Gamma$ & Transition matrix, control matrix, and training-action covariance. \\
$\eta,\ W$ & Process-noise scale and noise-shape covariance. \\
$R=\Sigma^{-1/2}W\Sigma^{-1/2}$ & State-whitened process-noise covariance. \\
$K=(A\Sigma^{1/2})^\top(A\Sigma^{1/2})$ & Encoder Gram matrix in state-whitened coordinates. \\
$O$ & Orthogonal factor in $A\Sigma^{1/2}=OK^{1/2}$. \\
$N,\ n_{\rm proj}$ & Batch size per regularized time slice and number of projection directions. \\
$\boldsymbol\omega$ & Collection of sampled SIGReg projection directions. \\
$\widehat{\mathcal R}_N,\ \mathcal R_N$ & Finite-batch and expected finite-batch SIGReg statistics. \\
$q$ & Scalar such that $qI_D$ minimizes the expected finite-batch statistic. \\
$\Lambda=\diag(v_1,\ldots,v_D)$ & Learnable diagonal Gaussian target covariance. \\
$\kappa$ & Condition-number bound on $\Lambda$; $\kappa=1$ gives the isotropic target. \\
$\mathcal T_{D,\kappa}$ & Feasible diagonal target family. \\
$L=O^\top\Lambda O$ & Target geometry in state-whitened coordinates. \\
$\mathcal S_{D,\kappa}$ & Feasible oriented target family. \\
$\lambda,\lambda_B,\lambda_T$ & Regularization weights; subscripts distinguish the isotropic- and learned-target theory. \\
$\mathcal U,\ y$ & Feasible action-sequence set and goal used in the explicit construction. \\
$\mu_{\eta,H}(U)$ & Conditional mean of the terminal physical state. \\
$\Omega_{\eta,H}$ & Terminal process-noise covariance. \\
$J_{*,\eta,H}(U)$ & Expected terminal task cost. \\
$\Regret_{\eta,H}(U)$ & Excess task cost over the best feasible sequence. \\
$\Delta_H$ & Positive limiting regret gap in the separation construction. \\
$\bar u,\ r_H$ & Per-step action bound and limiting reachable-ball radius. \\
$J_{\rm pred}(U)$ & Predicted terminal latent cost available to the planner. \\
$J_{\rm enc}(U)$ & Encoded cost of the realized reached observation. \\
$\ell_j,\ \pi_j$ & Feature-covariance eigenvalue and its normalized spectral weight. \\
$r_{\rm ent}$ & Entropy effective rank of the measured feature covariance. \\
\bottomrule
\end{tabular}
\end{table}
\clearpage
\section{Technical Preliminaries and a Two-Dimensional Closed-Form Specialization}
\label{app:proofs}

This appendix collects the finite-batch SIGReg facts used throughout the analysis and gives a two-dimensional specialization in which the selected action and planning regret are available in closed form.
The general joint-optimum and finite-horizon results used in the main text are proved in Appendix~\ref{ag:appendix}.
All matrix square roots are symmetric positive-semidefinite roots.
Matrix derivatives use the Frobenius inner product on symmetric matrices.
Every limit as $\eta\downarrow0$ fixes the batch size, quadrature, noise-shape parameters, action variance, and strictly positive regularization weights.
Scaling a regularization weight with $\eta$ defines a different asymptotic regime.

\subsection{Latent metric and isotropic covariance}
\label{app:metric}

\begin{proof}[Proof of Lemma~\ref{lem:metric}]
The covariance identity
\[
 A\Sigma A^\top=sI
\]
gives
\[
 \Sigma=sA^{-1}A^{-\top},
\]
and hence
\[
 A^\top A=s\Sigma^{-1}.
\]
Therefore
\[
 \norm{A(x-x_g)}^2
 =(x-x_g)^\top A^\top A(x-x_g)
 =s(x-x_g)^\top\Sigma^{-1}(x-x_g).
\]
Finally, two positive-definite quadratic forms agree up to a common
positive scale for every residual if and only if their symmetric
matrices are proportional.
Thus the latent Euclidean cost agrees up to positive scale with
$\Delta x^\top Q\Delta x$ for all residuals if and only if
$Q$ is proportional to $\Sigma^{-1}$.
\end{proof}

\paragraph{Ranking preservation under linear reparameterization.}
For completeness, consider an invertible linear transform $T$ applied
to arbitrary residuals.
It preserves all pairwise Euclidean cost rankings, including ties, if
and only if
\[
 T^\top T=cI
\]
for some $c>0$.

To see this, let $M=T^\top T\succ0$.
If $M=cI$, then
\[
 \norm{Te}^2=c\norm{e}^2
\]
for every residual $e$, so every ranking and tie is preserved.
Conversely, preservation of all ties requires $v^\top Mv$ to be
constant over the Euclidean unit sphere.
Evaluating the quadratic form at the eigenvectors of $M$ shows that
all eigenvalues of $M$ must coincide, and hence $M=cI$.

\paragraph{Ranking reversals between nonproportional metrics.}
More generally, let $M,Q\succ0$ be nonproportional.
Then there exist residuals $e_1,e_2$ whose ordering is reversed by the
two quadratic costs.
Indeed, let $v_1,v_2$ be orthonormal eigenvectors of
$M^{-1/2}QM^{-1/2}$ with eigenvalues $\mu_1<\mu_2$.
Choose $1<s<\mu_2/\mu_1$ and set
\[
 e_1=\sqrt{s}\,M^{-1/2}v_1,
 \qquad
 e_2=M^{-1/2}v_2.
\]
Then
\[
 e_1^\top M e_1=s>1=e_2^\top M e_2,
\]
whereas
\[
 e_1^\top Q e_1=s\mu_1<\mu_2=e_2^\top Q e_2.
\]
Thus nonproportional metrics can induce different rankings of feasible
outcomes.
The finite-horizon construction in Appendix~\ref{ag:finite-horizon}
realizes this geometric ranking mismatch through feasible action
sequences.

The affine-encoder assumption in Lemma~\ref{lem:metric} is essential.
In Cartesian coordinates, let $\operatorname{Rot}(\psi)$ be the planar rotation through angle $\psi$.
The smooth map $x\mapsto \operatorname{Rot}(\norm{x}^2)x$, or $(r,\theta)\mapsto(r,\theta+r^2)$ in polar coordinates, preserves standard two-dimensional Gaussian measure: the radius is unchanged, and the conditional angle remains uniform.
Its Cartesian Jacobian is nonorthogonal away from the origin.
A Gaussian marginal therefore places weaker restrictions on the local metric of a nonlinear encoder.

\subsection{Expected finite-batch SIGReg}
\label{sec:finite}
\label{app:finite}

Let $Z=(z_1,\ldots,z_N)$ contain $N>2$ independent samples from $\mathcal N(0,C)$.
Let $\omega_1,\ldots,\omega_{n_{\rm proj}}$, with $n_{\rm proj}\ge1$, be independent uniform directions on $\Sph^{D-1}$, independent of the batch.
For finitely many nonnegative knots $t_k$ and weights $w_k$, define
\begin{equation}
 \widehat{\mathcal R}_N(Z;\boldsymbol\omega)
 =\frac1{n_{\rm proj}}\sum_{\ell=1}^{n_{\rm proj}}\sum_k w_kN
 \left|\frac1N\sum_{b=1}^N e^{it_k\omega_\ell^\top z_b}
              -e^{-t_k^2/2}\right|^2,
 \label{eq:stat}
\end{equation}
where $\boldsymbol\omega=(\omega_1,\ldots,\omega_{n_{\rm proj}})$.
Assume that at least one positive-weight knot is strictly positive.
Averaging over samples and directions gives
\begin{equation}
 \begin{aligned}
 \Reg(C)&=\E_v r_N(v^\top Cv),\\
 r_N(s)&=\sum_k w_k\bigl[1+(N-1)e^{-t_k^2s}
       -2Ne^{-t_k^2(s+1)/2}+Ne^{-t_k^2}\bigr],
 \end{aligned}
 \label{eq:r}
\end{equation}
where $v$ is uniform on $\Sph^{D-1}$.
We assume $r_N'(0)<0$, a condition determined by the batch size and quadrature.

\begin{lemma}[Covariance minimum of the expected statistic]
\label{lem:finite}
There is a unique $q\in(0,1)$ satisfying $r_N'(q)=0$, and $\Reg$ is uniquely minimized over $C\succeq0$ at $qI_D$.
With $c_N=r_N''(q)>0$, its Hessian satisfies
\begin{equation}
 D^2\Reg(qI_D)[E,E]
 =\frac{c_N}{D(D+2)}
 \left[(\tr E)^2+2\tr(E^2)\right]
 \label{eq:hessian}
\end{equation}
for every symmetric matrix $E$.
\end{lemma}

The contraction from $I_D$ to $qI_D$ reflects finite-batch bias in the squared discrepancy between the empirical and reference characteristic functions \citep{balestriero2025lejepa}.
The distinction between population discrepancies and their finite-sample approximations also appears in kernel formulations of SIGReg \citep{zimmermann2025kerjepa}.
The lemma characterizes Gaussian covariances after averaging over samples and directions.

\begin{proof}[Proof of Lemma~\ref{lem:finite}]
For a fixed direction with projected variance $s=v^\top Cv\ge0$, let $Y_b=e^{it v^\top z_b}$ and $\widehat\phi=N^{-1}\sum_bY_b$.
Independence gives
\begin{equation}
 \E|\widehat\phi|^2
 =\frac1N+\frac{N-1}{N}e^{-t^2s},\qquad
 \E\widehat\phi=e^{-t^2s/2}.
\end{equation}
Expanding $N\E|\widehat\phi-e^{-t^2/2}|^2$ yields~\eqref{eq:r}, including at $s=0$.
This is the finite-sample expectation of the original biased empirical statistic used by SIGReg \citep{balestriero2025lejepa}.

Writing $a_k=t_k^2$, differentiation gives
\begin{align}
 r_N'(s)
 &=\sum_k w_ka_k\left[
 -(N-1)e^{-a_ks}+Ne^{-a_k(s+1)/2}\right],\\
 r_N''(s)
 &=\sum_k w_ka_k^2\left[
 (N-1)e^{-a_ks}-\tfrac N2e^{-a_k(s+1)/2}\right].
\end{align}
For $s\in[0,1]$, the bracket in the second line equals
\[
 e^{-a_k(s+1)/2}
 \left[(N-1)e^{a_k(1-s)/2}-\tfrac N2\right].
\]
It is strictly positive when $a_k>0$, since $N>2$.
Thus $r_N'$ is strictly increasing on $[0,1]$.
The assumption $r_N'(0)<0$ and the identity
\[
 r_N'(1)=\sum_k w_ka_ke^{-a_k}>0
\]
give a unique zero $q\in(0,1)$.
For $s\ge1$, the bracket in the first derivative is at least $e^{-a_k(s+1)/2}$, so $r_N'(s)>0$.
Consequently, $q$ is the unique global minimizer of $r_N$ on $[0,\infty)$, and $c_N=r_N''(q)>0$.

Pointwise minimization gives $\Reg(C)\ge r_N(q)$.
Equality requires $v^\top Cv=q$ for almost every unit vector $v$.
Continuity extends this identity to every unit vector, which implies $C=qI_D$.
Conversely, $qI_D$ attains equality.
This proves uniqueness over the entire positive-semidefinite cone, including singular covariances.

The quadrature is finite and the sphere is compact, so the integrand's matrix derivatives are uniformly bounded on each compact covariance set.
Differentiation under the expectation is therefore justified by dominated convergence.
At $qI_D$ this gives
\[
 D^2\Reg(qI_D)[E,E]=c_N\E(v^\top Ev)^2.
\]
The uniform-sphere fourth moments are
\[
 \E v_i^4=\frac{3}{D(D+2)},\qquad
 \E v_i^2v_j^2=\frac{1}{D(D+2)}\quad(i\ne j).
\]
Diagonalizing $E$ and expanding the square proves~\eqref{eq:hessian}.
\end{proof}

LeWM's SIGReg statistic~\citep{maes2026lewm} uses knots $t_k=3k/16$, $k=0,\ldots,16$, and the weights are $w_k=c_k(3/16)e^{-t_k^2/2}$, with $c_0=c_{16}=1$ and $c_k=2$ otherwise.
These are symmetry-doubled trapezoidal weights on $[0,3]$, multiplied by the Gaussian window.
A sufficient condition for $r_N'(0)<0$ is $N\ge58$.
For every positive knot,
\begin{equation}
 1-e^{-t_k^2/2}
 \ge\frac{t_k^2/2}{1+t_k^2/2}
 \ge\frac9{521},\qquad
 1-N(1-e^{-t_k^2/2})<0.
\end{equation}
The first inequality follows from $e^x\ge1+x$, and the final strict inequality follows from $58\cdot9>521$.
Every nonzero summand of $r_N'(0)$ is therefore negative.
The sufficient condition includes $N=128$, the batch size at each time point in our training setup.
A smaller batch threshold may also suffice.

Independent normalized Gaussian directions are uniform on the sphere, as required by the calculation.
We compute the empirical characteristic function across examples at each time point and average the resulting statistics over time.
Each training batch contains four time points with $N=128$ examples each.
If these slices have the same Gaussian marginal covariance $C$ and independent examples within each slice, the expected average remains $\Reg(C)$ under dependence across slices.
With different covariances $C_t$, the expectation is the average of $\Reg(C_t)$; dependence among examples within a slice requires a different finite-sample calculation.

\subsection{Two-dimensional closed-form specialization}
\label{app:two_dimensional}
\label{sec:obstruction}

The general results in Section~\ref{sec:analysis} do not require a two-dimensional state space.
The following specialization is useful because both the isotropic and learned-target planners admit closed-form limiting actions and regrets.

Consider independent reset transitions
\begin{equation}
 \begin{aligned}
 x'&=F_\eta x+gu+\xi,\\
 x&\sim\mathcal N(0,\Sigma),\qquad
 u\sim\mathcal N(0,\tau^2),\qquad
 \xi\sim\mathcal N(0,W_\eta),
 \end{aligned}
 \label{eq:world}
\end{equation}
where $x,u,\xi$ are mutually independent.
Set
\[
 \Sigma=\diag(1+\delta,1-\delta),\qquad
 R_0=\diag(\rho_1,\rho_2),\qquad
 g=(1,1)^\top,
\]
with $0<\delta<1$, $\rho_1,\rho_2>0$, and $0<\tau^2<(1-\delta^2)/2$, and let
\[
 W_\eta=\eta\Sigma^{1/2}R_0\Sigma^{1/2}.
\]
The covariance-preserving transition matrix $F_\eta$ is defined in Eq.~\eqref{eq:transition_construction} below.
For sufficiently small $\eta>0$, it is invertible and gives $x'$ the same covariance $\Sigma$ as $x$.

For evaluation, let $d=(1,-1)^\top$, take $x_0=F_\eta^{-1}d$, set $x_g=0$, and restrict $u\in[-1,1]$.
The terminal conditional mean is
\[
 d+gu=(1+u,u-1)^\top.
\]
We evaluate actions using squared Euclidean distance in state space:
\begin{equation}
 \begin{aligned}
 J_*(u)&=\E\norm{d+gu+\xi}^2
       =2+2u^2+\tr W_\eta,\\
 u_*&=0,\qquad
 \Regret(u)=J_*(u)-J_*(0)=2u^2.
 \end{aligned}
 \label{eq:task}
\end{equation}
The noise contributes the same additive term to every action and therefore does not affect their ranking.

We train linear encoders $A\in\R^{2\times2}$, including singular encoders, together with predictors linear in $(Ax,u)$ under
\begin{equation}
 \mathcal L_B(A,p)
 =\underbrace{\tfrac12\E\norm{p(Ax,u)-Ax'}^2}_{\Lpred(A,p)}
 +\lambda_B\Reg(A\Sigma A^\top),
 \qquad \lambda_B>0.
 \label{eq:baseline}
\end{equation}
The latent planner minimizes
\[
 J_{z,\eta}(u)=\norm{p(Ax_0,u)-Ax_g}^2
\]
over the same feasible interval.

\begin{theorem}[Closed-form isotropic-target specialization]
\label{thm:obstruction}
Fix the preceding control parameters and the statistic assumptions of Lemma~\ref{lem:finite}, and let $\lambda_B>0$.
For all sufficiently small $\eta>0$, Eq.~\eqref{eq:baseline} attains a global minimum over all linear encoders and predictors.
Every minimizing encoder is invertible, and every minimizing predictor recovers the exact encoded conditional mean,
\[
 p(Ax,u)=A(F_\eta x+gu).
\]
All global minimizers induce the same
\[
 K_\eta=(A\Sigma^{1/2})^\top(A\Sigma^{1/2})\longrightarrow qI,
\]
and therefore
\begin{equation}
 M_\eta=A^\top A
 =\Sigma^{-1/2}K_\eta\Sigma^{-1/2}
 \longrightarrow q\Sigma^{-1}.
 \label{eq:baseline_metric}
\end{equation}
If $u_B$ minimizes $J_{z,\eta}$ on $[-1,1]$, then
\begin{equation}
 \Lpred(A,p)=\tfrac\eta2\tr(K_\eta R_0)\longrightarrow0,
 \qquad
 u_B\longrightarrow\delta,
 \qquad
 \Regret(u_B)\longrightarrow2\delta^2>0.
 \label{eq:obstruction}
\end{equation}
For all sufficiently small positive $\eta$, the latent cost strictly prefers $u_B$ to the task-optimal action $0$, whereas the task cost strictly prefers $0$ to $u_B$.
\end{theorem}

Equation~\eqref{eq:baseline_metric} makes the mechanism explicit: the lower-variance second state coordinate receives greater weight under $\Sigma^{-1}$, so the latent planner trades the two terminal errors differently from the equal-weight task cost.
This specialization is not needed for the general separation in Section~\ref{sec:finite_horizon_separation}; it is used in Appendix~\ref{app:targetfamily} to obtain a closed-form dependence on the target anisotropy bound.

\subsection{Transition construction and prediction reduction for the specialization}
\label{app:reduction}

For the independent reset model in~\eqref{eq:world}, fix $0<\delta<1$, $\rho_1,\rho_2>0$, and $0<\tau^2<(1-\delta^2)/2$.
With the training-state covariance $\Sigma=\diag(1+\delta,1-\delta)$, noise shape $R_0=\diag(\rho_1,\rho_2)$, and $g=(1,1)^\top$, set
\begin{equation}
 W_\eta=\eta\Sigma^{1/2}R_0\Sigma^{1/2},\qquad
 F_\eta=(\Sigma-\tau^2gg^\top-W_\eta)^{1/2}\Sigma^{-1/2}.
 \label{eq:transition_construction}
\end{equation}
The model is well defined for sufficiently small $\eta$.
The rank-one positive-definiteness criterion gives
\begin{equation}
 \begin{aligned}
 \Sigma-\tau^2gg^\top\succ0
 &\quad\Longleftrightarrow\quad
 \tau^2g^\top\Sigma^{-1}g<1\\
 &\quad\Longleftrightarrow\quad
 \tau^2<\frac{1-\delta^2}{2}.
 \end{aligned}
\end{equation}
Positive definiteness persists after subtracting $W_\eta$ for small $\eta>0$.
The resulting $F_\eta$ is invertible and satisfies
\[
 F_\eta\Sigma F_\eta^\top+\tau^2gg^\top+W_\eta=\Sigma.
\]
Thus $x$ and $x'$ have the same Gaussian marginal.

The same covariance identity gives a stationary trajectory construction.
Initialize $x_0\sim\mathcal N(0,\Sigma)$ and draw iid innovation pairs $(u_t,\xi_t)\sim\mathcal N(0,\tau^2)\otimes\mathcal N(0,W_\eta)$, with the pair sequence independent of $x_0$.
Induction on the transition gives $x_t\sim\mathcal N(0,\Sigma)$ at every time.
Each transition has the same one-step tuple law as~\eqref{eq:world}, so $N$ independent trajectories provide iid examples within each time slice.

For fixed encoder, predictor, and target parameters, the expected prediction loss averaged over any fixed finite set of times equals that of the reset construction.
The same equality holds for the average of SIGReg statistics computed separately at each time slice.

For any encoder $A$, including singular $A$, independence and the zero mean of $\xi$ yield
\begin{align}
 \Lpred(A,p)
 &=\tfrac12\E\norm{p(Ax,u)-AF_\eta x-Agu}^2
   +\tfrac12\tr(AW_\eta A^\top)\nonumber\\
 &\ge\tfrac\eta2\tr(KR_0),\qquad
 B=A\Sigma^{1/2},\quad K=B^\top B\succeq0.
 \label{eq:predictionbound}
\end{align}
For invertible $A$, equality is attained by the linear encoded conditional-mean predictor
\begin{equation}
 p_A(z,u)=AF_\eta A^{-1}z+Agu.
 \label{eq:bayes}
\end{equation}
Equality uniquely determines the predictor.
The covariance of $(Ax,u)$ is positive definite, so two linear predictions that agree almost surely have identical coefficient matrices.

The matrices $BB^\top$ and $B^\top B$ are orthogonally conjugate, including when $B$ is singular.
Orthogonal invariance of $\Reg$ therefore gives $\Reg(A\Sigma A^\top)=\Reg(K)$.
The loss of every encoder and predictor pair is bounded below by the reduced objective
\begin{equation}
 \mathcal F_\eta(K)
 =\tfrac\eta2\tr(KR_0)+\lambda_B\Reg(K),\qquad K\succeq0.
 \label{eq:Freduced}
\end{equation}
Every positive-definite $K$ attains this bound with $A=K^{1/2}\Sigma^{-1/2}$ and the predictor in~\eqref{eq:bayes}.
It remains to show that all reduced global minimizers are positive definite.

\subsection{Proof of Theorem~\ref{thm:obstruction}}
\label{app:obstruction}

\paragraph{Existence and uniform localization.}
For every fixed $\eta>0$, $\mathcal F_\eta$ is continuous and coercive on the positive-semidefinite cone, since $R_0\succ0$ and $\Reg\ge0$.
It therefore attains its minimum.
Comparing any minimizer $K_\eta$ with $qI$ gives
\begin{equation}
 \frac\eta2\tr(K_\eta R_0)
 +\lambda_B\bigl[\Reg(K_\eta)-r_N(q)\bigr]
 \le\frac{\eta q}{2}\tr(R_0).
 \label{eq:comparison}
\end{equation}
Both terms on the left are nonnegative.
In particular,
\[
 \tr(K_\eta R_0)\le q\tr(R_0),\qquad
 0\le\Reg(K_\eta)-r_N(q)
 \le\frac{\eta q}{2\lambda_B}\tr(R_0).
\]
The trace bound is uniform over $\eta>0$ and all reduced global minimizers.
Every cluster point as $\eta\downarrow0$ minimizes $\Reg$, so Lemma~\ref{lem:finite} gives $K_\eta\to qI$.
To verify uniform convergence over the minimizing sets, suppose a sequence of minimizers lies outside a fixed neighborhood of $qI$.
The trace bound gives a convergent subsequence whose limit must be $qI$, a contradiction.

\paragraph{Uniqueness and the first-order perturbation.}
The Hessian in~\eqref{eq:hessian} is positive definite at $qI$ and, by continuity, throughout a sufficiently small convex neighborhood.
All global minimizers eventually lie in this neighborhood.
Strict convexity therefore gives a unique positive-definite reduced minimizer $K_\eta$ for small $\eta>0$.
Conjugation by $J=\diag(1,-1)$ preserves both terms of~\eqref{eq:Freduced}.
Uniqueness implies $JK_\eta J=K_\eta$, so $K_\eta$ is diagonal.

The stationarity equation is
\[
 \frac\eta2R_0+\lambda_B\nabla\Reg(K_\eta)=0.
\]
In two dimensions, the Hessian operator at $qI$ is
\begin{equation}
 E\longmapsto\frac{c_N}{8}\bigl[\tr(E)I+2E\bigr].
\end{equation}
It is invertible on the symmetric matrices.
The implicit function theorem gives a smooth local solution, which coincides with the unique global minimizer for small positive $\eta$.
Differentiating at zero yields
\begin{equation}
 K'_0=-\frac{1}{\lambda_Bc_N}
       \left(2R_0-\bar\rho I\right),\qquad
 \bar\rho=\frac{\rho_1+\rho_2}{2}.
\end{equation}
Consequently,
\begin{equation}
 K_\eta=qI-\frac{\eta}{\lambda_Bc_N}
 \left(2R_0-\tfrac12\tr(R_0)I\right)+O(\eta^2).
 \label{eq:Kexp}
\end{equation}
Write $K_\eta=\diag(c_1,c_2)$ and $\beta_\eta=(c_1-c_2)/(c_1+c_2)$.
Then
\begin{equation}
 \beta_\eta
 =\frac{\eta(\rho_2-\rho_1)}{\lambda_Bqc_N}
 +O(\eta^2).
 \label{eq:betaexp}
\end{equation}
For sufficiently small positive $\eta$, convergence gives $|\beta_\eta|<\delta$ and $c_i>q/2$.
Under the additional ordering $\rho_1<\rho_2$, the expansion also gives $\beta_\eta>0$ and $c_1\rho_1-c_2\rho_2<0$.
If $\rho_1=\rho_2$, the objective is invariant under every orthogonal conjugation, so uniqueness forces $K_\eta$ to be a scalar matrix and $\beta_\eta=0$.

The positive-definite reduced minimizer is attained by an invertible encoder and~\eqref{eq:bayes}, so the full training objective has a global minimum.
Every full minimizer must attain both the reduced minimum and the prediction bound.
Failure to attain either would give a strictly larger objective value.
All full global minimizers therefore induce the same $K_\eta$ and recover the exact encoded conditional mean.
The feature covariance $BB^\top$ has eigenvalues $c_1,c_2>q/2$, proving the noncollapse bound.

\paragraph{Planning and feasible regret.}
At $x_0$, the encoded conditional mean is $A(d+gu)$.
The latent planner therefore minimizes
\begin{equation}
 \begin{aligned}
 J_{z,\eta}(u)&=(d+gu)^\top M_\eta(d+gu),\\
 M_\eta&=A^\top A=\Sigma^{-1/2}K_\eta\Sigma^{-1/2}
       =\diag\left(\frac{c_1}{1+\delta},
                    \frac{c_2}{1-\delta}\right).
 \end{aligned}
\end{equation}
Writing the diagonal entries as $m_1,m_2>0$, the objective is $m_1(1+u)^2+m_2(u-1)^2$.
Its unique unconstrained minimizer is
\[
 u_B=\frac{m_2-m_1}{m_1+m_2}
     =\frac{\delta-\beta_\eta}{1-\delta\beta_\eta}
     \in(0,1),
\]
which is feasible for sufficiently small $\eta$.
When $\rho_1<\rho_2$, it lies in $(0,\delta)$.
Equation~\eqref{eq:task} gives $\Regret(u_B)=2u_B^2\to2\delta^2$.
The attained prediction loss is $\eta\tr(K_\eta R_0)/2$ and tends to zero.
Strict convexity and $u_B\ne0$ imply $J_{z,\eta}(u_B)<J_{z,\eta}(0)$ and $J_*(0)<J_*(u_B)$.
This proves the ranking reversal.
\qed

The strict inequalities between $u_B$ and $0$ extend by continuity to sufficiently small feasible neighborhoods of those actions.
A candidate pool containing actions from both neighborhoods therefore reverses their ordering under the latent and physical task costs.
Finite-budget CEM performance also depends on candidate sampling and the search procedure.

\subsection{General residuals and a fixed evaluation state}
\label{app:residuals}

Let the terminal conditional-mean residual be $\tilde d=(d_1,d_2)^\top$, with goal zero, $g=(1,1)^\top$, and task metric $I$.
For a diagonal latent metric $M=\diag(m_1,m_2)\succ0$, the unconstrained task and latent optima are
\[
 u_*^0=-\frac{d_1+d_2}{2},\qquad
 u_z^0=-\frac{m_1d_1+m_2d_2}{m_1+m_2}.
\]
Subtracting and completing the square in the task cost gives
\begin{equation}
 u_z^0-u_*^0
 =\frac{(m_2-m_1)(d_1-d_2)}{2(m_1+m_2)},\qquad
 J_*(u)-J_*(u_*^0)=2(u-u_*^0)^2.
 \label{eq:general_residual}
\end{equation}
When both optima are interior to $[-1,1]$, the resulting regret is
\[
 \Regret(u_z^0)
 =\frac{(d_1-d_2)^2}{2}
   \left(\frac{m_2-m_1}{m_1+m_2}\right)^2.
\]
For the baseline metric, this converges to $\delta^2(d_1-d_2)^2/2$.
For the favorable learned-target limit in Theorem~\ref{thm:bounded}, the corresponding limit is
\begin{equation}
 \frac{(d_1-d_2)^2}{2}
 \left(\frac{c-\kappa}{c+\kappa}\right)^2.
 \label{eq:general_target_regret}
\end{equation}
Assume that the task, baseline, and target limiting optima are interior and that $d_1\ne d_2$.
For each fixed residual satisfying these conditions, the target limiting regret is strictly lower than the baseline limit exactly when $1<\kappa<c^2$.
A uniform positive improvement margin over a set of residuals further requires a positive lower bound on $|d_1-d_2|$.

For arbitrary residuals, let $\Pi(t)=\min\{1,\max\{-1,t\}\}$.
The constrained optima are $u_*=\Pi(u_*^0)$ and $u_z=\Pi(u_z^0)$, and the exact regret is
\[
 \Regret(u_z)=2\bigl[(u_z-u_*^0)^2-(u_*-u_*^0)^2\bigr].
\]
Distinct unconstrained optima can project to the same endpoint.
For example, $(m_1,m_2)=(1,2)$ and $(d_1,d_2)=(2,3)$ give $u_*^0=-5/2$ and $u_z^0=-8/3$, both projecting to $-1$.
The strictness statement therefore uses the interior assumption.

The original evaluation state $x_0=F_\eta^{-1}d$ fixes the residual $d=(1,-1)^\top$ at every noise level.
Alternatively, take the fixed state $\bar x_0=F_0^{-1}d$, where $F_0=(\Sigma-\tau^2gg^\top)^{1/2}\Sigma^{-1/2}$.
Continuity of the positive-definite square root gives $F_\eta\bar x_0\to d$.
The projected quadratic argmin is continuous in the residual and the positive-definite metric, so the actions and regrets have the limits in Theorems~\ref{thm:obstruction} and~\ref{thm:bounded}.
At positive noise, regret in this variant is measured relative to its task-optimal action $\Pi(-g^\top F_\eta\bar x_0/2)$, which tends to zero.

\section{Closed-Form Analysis of the Learnable Gaussian Target}
\label{app:targetfamily}

This appendix analyzes $\Lambda$Reg in the two-dimensional specialization of Appendix~\ref{app:two_dimensional}.
The purpose of the specialization is to expose the dependence of the selected action and regret on the anisotropy bound $\kappa$ in closed form.
The general learned-target metric is treated in Appendix~\ref{ag:target-geometry}, and arbitrary-dimensional spectrum selection is treated in Appendix~\ref{app:general_spectrum}.

Assume the two-dimensional model and evaluation setup of Theorem~\ref{thm:obstruction}.
For a diagonal target $\Lambda\in\mathcal T_{2,\kappa}$, the expected learned-target objective is
\begin{equation}
 \mathcal L_T(A,p,\Lambda)
 =\Lpred(A,p)
  +\lambda\Reg\!\left(
    \Lambda^{-1/2}A\Sigma A^\top\Lambda^{-1/2}
  \right),
 \qquad \Lambda\in\mathcal T_{2,\kappa}.
 \label{eq:targetobjective}
\end{equation}
We focus first on the ordering $\rho_1<\rho_2$, under which predictive training favors a variance reallocation that counteracts the inverse-covariance weighting of the isotropic solution.
The reverse ordering is treated in Appendix~\ref{app:scope}.

\begin{theorem}[Selected geometry and planning regret in the two-dimensional specialization]
\label{thm:bounded}
Under the assumptions of Theorem~\ref{thm:obstruction}, suppose $\rho_1<\rho_2$, fix $\kappa>1$ and $\lambda>0$, and define
\[
 c=\cond(\Sigma)=\frac{1+\delta}{1-\delta},
 \qquad
 b=\frac{\kappa-1}{\kappa+1}.
\]
For all sufficiently small $\eta>0$, Eq.~\eqref{eq:targetobjective} attains a global minimum.
Every minimizing encoder is invertible, and every minimizing predictor recovers the exact encoded conditional mean.
For each global minimizer, let
\[
 K_{\eta,T}=(A\Sigma^{1/2})^\top(A\Sigma^{1/2}),
\]
and let $u_T$ minimize the learned latent cost over $[-1,1]$.
Then, as $\eta\downarrow0$,
\begin{equation}
 K_{\eta,T}\longrightarrow q\diag(1+b,1-b),
 \qquad
 u_T\longrightarrow\frac{c-\kappa}{c+\kappa},
 \qquad
 \Regret(u_T)\longrightarrow
 2\left(\frac{c-\kappa}{c+\kappa}\right)^2.
 \label{eq:targetlimit}
\end{equation}
The limiting regret is zero at $\kappa=c$.
Relative to the isotropic baseline at any fixed $\lambda_B>0$, it is strictly lower if and only if $1<\kappa<c^2$, equal at $\kappa=c^2$, and strictly higher for $\kappa>c^2$.
\end{theorem}

At $\kappa=c$, the limiting state-space metric is $qI$, so the learned target exactly cancels the inverse-covariance weighting of the isotropic solution in this specialization.
The improvement is not monotone in $\kappa$: excessive anisotropy eventually increases regret.
This specialization makes the effect of the anisotropy bound explicit; Appendix~\ref{app:adaptive_spectrum} characterizes the target spectra selected in higher dimensions.

\subsection{Proof of Theorem~\ref{thm:bounded}}
\label{app:bounded}

The prediction-loss normalization is the same as in Appendix~\ref{app:two_dimensional}, so the reduced prediction bound in Eq.~\eqref{eq:predictionbound} applies directly.

\paragraph{Reduction with a free encoder orientation.}
For independent samples $x_b\sim\mathcal N(0,\Sigma)$ and $Z=(Ax_1,\ldots,Ax_N)$, the $\Lambda$Reg statistic satisfies
\[
 \E_{Z,\boldsymbol\omega}\widehat{\mathcal R}_{N,\Lambda}(Z;\boldsymbol\omega)
 =\Reg\left(
 \Lambda^{-1/2}A\Sigma A^\top\Lambda^{-1/2}
 \right).
\]
This gives the expected objective in~\eqref{eq:targetobjective}.
To separate the target spectrum from the encoder orientation, let
\[
 \mathcal S_\kappa
 =\{L\succ0:\tr L=2,\ \cond(L)\le\kappa\}
\]
be the family of orthogonal rotations of the diagonal target covariances.
Write $B=A\Sigma^{1/2}=OK^{1/2}$ by polar decomposition, with $K=B^\top B$ and an orthogonal extension when $B$ is singular.
Set $L=O^\top\Lambda O\in\mathcal S_\kappa$.
Since $\Lambda^{-1/2}=OL^{-1/2}O^\top$,
\[
 \begin{aligned}
 \Lambda^{-1/2}A\Sigma A^\top\Lambda^{-1/2}
 &=\Lambda^{-1/2}BB^\top\Lambda^{-1/2}\\
 &=O\bigl(L^{-1/2}KL^{-1/2}\bigr)O^\top.
 \end{aligned}
\]
Orthogonal invariance of $\Reg$ and the prediction lower bound in~\eqref{eq:predictionbound} give the reduced lower bound
\begin{equation}
 \begin{aligned}
 \mathcal G_\eta(K,L)
 &=\frac{\eta}{2}\tr(KR_0)
   +\lambda\Reg\bigl(L^{-1/2}KL^{-1/2}\bigr),\\
 &\hspace{2em}K\succeq0,\qquad L\in\mathcal S_\kappa.
 \end{aligned}
 \label{eq:Greduced}
\end{equation}
Conversely, every admissible pair $(K,L)$ with $K\succ0$ and $L\in\mathcal S_\kappa$ is realizable.
Choose an orthogonal $O$ such that $OLO^\top=\Lambda$ is diagonal, set $A=OK^{1/2}\Sigma^{-1/2}$, and use the predictor in~\eqref{eq:bayes}.
This pair attains the reduced bound.
The encoder's orthogonal freedom realizes every orientation in $\mathcal S_\kappa$ with a diagonal output target.

\paragraph{The unique limiting allocation.}
For $b=(\kappa-1)/(\kappa+1)$, every $L\in\mathcal S_\kappa$ has the form
\begin{equation}
 L=I+\begin{pmatrix}s&t\\t&-s\end{pmatrix},
 \qquad s^2+t^2\le b^2.
 \label{eq:disk}
\end{equation}
Its eigenvalues are $1\pm\sqrt{s^2+t^2}$, whose ratio is at most $\kappa$ precisely on this disk.
Moreover,
\[
 \tr(LR_0)=\rho_1+\rho_2-(\rho_2-\rho_1)s.
\]
Since $\rho_2>\rho_1$, the unique minimum over the disk occurs at $s=b,t=0$, giving $L_*=\diag(1+b,1-b)$.

For fixed $\eta>0$, $\mathcal G_\eta$ attains its minimum.
The set $\mathcal S_\kappa$ is compact with a positive lower eigenvalue bound, and the nonnegative prediction term bounds $\tr(KR_0)$ on every sublevel set.
Comparing a minimizing pair $(K_{\eta,T},L_\eta)$ with $(qL_*,L_*)$ gives
\begin{equation}
 \begin{gathered}
 \frac{\eta}{2}\tr(K_{\eta,T}R_0)
 +\lambda\bigl[\Reg(C_\eta)-r_N(q)\bigr]
 \le\frac{\eta q}{2}\tr(L_*R_0),\\
 C_\eta=L_\eta^{-1/2}K_{\eta,T}L_\eta^{-1/2}.
 \end{gathered}
 \label{eq:Gcomparison}
\end{equation}
Both terms on the left are nonnegative, so $\tr(K_{\eta,T}R_0)\le q\tr(L_*R_0)$ uniformly over all minimizers.
The lower eigenvalue bound on $L_\eta$ also bounds $C_\eta$.
The regularizer gap tends to zero, and every cluster point of $C_\eta$ equals $qI$ by Lemma~\ref{lem:finite}.
Hence $C_\eta\to qI$ uniformly over all minimizers.

For any subsequence with $L_\eta\to\bar L$,
\[
 K_{\eta,T}=L_\eta^{1/2}C_\eta L_\eta^{1/2}
 \longrightarrow q\bar L.
\]
The trace bound from~\eqref{eq:Gcomparison} implies $\tr(\bar L R_0)\le\tr(L_*R_0)$.
Uniqueness of the minimizer over the disk in~\eqref{eq:disk} forces $\bar L=L_*$.
Thus
\[
 L_\eta\longrightarrow L_*,
 \qquad K_{\eta,T}\longrightarrow qL_*
\]
uniformly over the minimizing sets.
Every minimizing $K_{\eta,T}$ is therefore positive definite for sufficiently small $\eta$, with
\[
 \lambda_{\min}(K_{\eta,T})>\frac{q(1-b)}{2}.
\]

The positive-definite reduced minimizers can all be realized by an encoder, predictor, and diagonal target, so the full objective attains the reduced minimum.
As in Appendix~\ref{app:obstruction}, every full global minimizer must attain both the reduced minimum and the prediction lower bound.
Its encoder is invertible and its predictor is $p_A$ from~\eqref{eq:bayes}.
The preceding uniform limits therefore hold for the Gram matrices of all full global minimizers.

\paragraph{The limiting action and the improvement range.}
For a positive-definite state metric $M$, the unconstrained minimizing action is
\[
 u(M)=-\frac{g^\top Md}{g^\top Mg},
\]
which is continuous in $M$.
The limiting state metric is
\[
 q\diag\left(
 \frac{1+b}{1+\delta},\frac{1-b}{1-\delta}
 \right),
\]
so
\[
 u_T\longrightarrow\frac{\delta-b}{1-\delta b}
 =\frac{c-\kappa}{c+\kappa},
 \qquad c=\frac{1+\delta}{1-\delta}.
\]
This limit lies strictly in $(-1,1)$ for $\delta,b\in(0,1)$.
The unconstrained action is therefore eventually feasible, uniformly over the minimizing sets.
Positive definiteness makes the quadratic cost strictly convex, so the minimizing action is unique.
The task cost in~\eqref{eq:task} gives the regret limit in~\eqref{eq:targetlimit}.

Relative to any baseline with fixed $\lambda_B>0$, the limiting regret improvement is
\begin{equation}
 \begin{aligned}
 \Delta_{\rm imp}(c,\kappa)
 &=2\left(\frac{c-1}{c+1}\right)^2
   -2\left(\frac{c-\kappa}{c+\kappa}\right)^2\\
 &=\frac{8c(c^2-\kappa)(\kappa-1)}{(c+1)^2(c+\kappa)^2}.
 \end{aligned}
 \label{eq:gain}
\end{equation}
Since $c>1$ and $\kappa>1$, the sign of $\Delta_{\rm imp}$ is the sign of $c^2-\kappa$.
Limiting regret is therefore lower than the baseline for $1<\kappa<c^2$, equal at $\kappa=c^2$, and higher for $\kappa>c^2$.
It vanishes at $\kappa=c$.
Uniform convergence over the minimizing sets of both objectives gives uniform convergence of their regret difference to $\Delta_{\rm imp}$.
When $\Delta_{\rm imp}>0$, choose $\eta$ sufficiently small that the deviation is less than $\Delta_{\rm imp}/2$.
The finite-noise improvement then exceeds $\Delta_{\rm imp}/2$.
At $\kappa=c^2$, higher-order terms determine the finite-noise sign.
\hfill$\square$

The limiting regret depends on $|\log\kappa-\log c|$, so it is symmetric about $\kappa=c$ on the logarithmic scale.

\subsection{A target-coupled rescaling path}
\label{sec:path}

To examine how target anisotropy trades prediction loss against planning regret, we construct a coupled rescaling of an isotropic-target optimum.
The path preserves the standardized features seen by $\Lambda$Reg while changing the encoder's representation geometry.
Let $c_1,c_2$ be the baseline Gram eigenvalues from Appendix~\ref{app:obstruction}.
We use the canonical baseline encoder
\[
 A_B=\diag(\sqrt{c_1},\sqrt{c_2})\Sigma^{-1/2}
\]
and predictor $p_B=p_{A_B}$ from~\eqref{eq:bayes} as the starting pair.
Define the encoder, predictor, and target along the path by
\begin{equation}
 \begin{aligned}
 \Lambda_h&=\diag(1+h,1-h),\qquad
 S_h=\Lambda_h^{1/2},\qquad 0\le h<1,\\
 A_h&=S_hA_B,\qquad
 p_h(z,u)=S_hp_B(S_h^{-1}z,u).
 \end{aligned}
 \label{eq:path}
\end{equation}
Each target has trace two.
Under a fixed condition-number bound $\kappa$, this path belongs to $\mathcal T_{2,\kappa}$ precisely for $0\le h\le b=(\kappa-1)/(\kappa+1)$.
The full interval $0\le h<1$ describes the effect of varying the allowed anisotropy.

Since $\Lambda_h^{-1/2}A_h=A_B$, $\Lambda$Reg is constant along this path: for every batch $X$ and common set of projection directions $\boldsymbol\omega$,
\[
 \widehat{\mathcal R}_{N,\Lambda_h}(A_hX;\boldsymbol\omega)
 =\widehat{\mathcal R}_N(A_BX;\boldsymbol\omega).
\]
For the evaluation state and goal in Appendix~\ref{app:two_dimensional}, let
\[
 u_h=\operatorname*{arg\,min}_{u\in[-1,1]}
 \norm{p_h(A_hx_0,u)-A_hx_g}^2.
\]

\begin{proposition}[A target-coupled rescaling path]
\label{prop:path}
Under the baseline setup of Theorem~\ref{thm:obstruction}, fix $\lambda_B>0$ and assume $\rho_1<\rho_2$.
For all sufficiently small $\eta>0$ and every $0\le h<1$,
\begin{align}
 \Lpred(A_h,p_h)-\Lpred(A_B,p_B)
 &=\frac{\eta h}{2}(c_1\rho_1-c_2\rho_2)<0
 \qquad(h>0),
 \label{eq:pathpred}\\
 \Regret(u_h)
 &=2\left(\frac{u_B-h}{1-u_Bh}\right)^2.
 \label{eq:pathregret}
\end{align}
Along the full path, regret strictly decreases on $[0,u_B]$, reaches zero at $h=u_B$, and exceeds its baseline value when $h>2u_B/(1+u_B^2)$.
At every $h>0$, the expected original isotropic-target regularizer is strictly larger than its baseline value.
Its increase, weighted by $\lambda_B$, exceeds the decrease in expected prediction loss.
\end{proposition}

Increasing $h$ continues to reduce latent prediction loss after regret reaches zero.
Each predictor along the path recovers the encoded conditional mean, so the change in~\eqref{eq:pathpred} comes from the encoder's weighting of process noise.
Proposition~\ref{prop:path} compares objectives and planning costs along this prescribed path in parameter space.
Theorem~\ref{thm:bounded} identifies the limiting geometry selected by the constrained joint objective.

\begin{proof}[Proof of Proposition~\ref{prop:path}]
The baseline predictor $p_B$ recovers the encoded conditional mean.
The definitions of $A_h$ and $p_h$ apply $S_h$ to its prediction and target, giving
\[
 \begin{aligned}
 \Lpred(A_h,p_h)
 &=\frac12\tr(A_hW_\eta A_h^\top)\\
 &=\frac{\eta}{2}
 \bigl[(1+h)c_1\rho_1+(1-h)c_2\rho_2\bigr].
 \end{aligned}
\]
Subtracting the baseline loss proves~\eqref{eq:pathpred}.
Appendix~\ref{app:obstruction} establishes $c_1\rho_1-c_2\rho_2<0$ for sufficiently small positive $\eta$.

Write the baseline state-metric entries as $m_1=c_1/(1+\delta)$ and $m_2=c_2/(1-\delta)$, so that $u_B=(m_2-m_1)/(m_1+m_2)$.
Along the path, these entries become $(1+h)m_1$ and $(1-h)m_2$.
The minimizing action is therefore
\[
 u_h=
 \frac{(1-h)m_2-(1+h)m_1}{(1+h)m_1+(1-h)m_2}
 =\frac{u_B-h}{1-u_Bh}.
\]
For $0\le h<1$, it lies strictly in $(-1,1)$ and satisfies
\[
 \frac{du_h}{dh}
 =-\frac{1-u_B^2}{(1-u_Bh)^2}<0.
\]
Thus $u_h^2$ strictly decreases until $h=u_B$ and strictly increases thereafter.
Since $\Regret(u)=2u^2$, this gives \eqref{eq:pathregret}.
Solving $u_h^2<u_B^2$ yields
\[
 0<h<\frac{2u_B}{1+u_B^2},
\]
with equality at $h=0$ and at the upper threshold.

The unstandardized Gram matrix along the path is
\[
 K_h=\diag\bigl((1+h)c_1,(1-h)c_2\bigr)\ne K_\eta
 \qquad(h>0).
\]
Uniqueness of the minimizer of~\eqref{eq:Freduced} implies
\[
 \begin{aligned}
 0&<\mathcal F_\eta(K_h)-\mathcal F_\eta(K_\eta)\\
  &=\Delta\Lpred
    +\lambda_B\bigl[\Reg(K_h)-\Reg(K_\eta)\bigr],
 \end{aligned}
\]
where $\Delta\Lpred=\Lpred(A_h,p_h)-\Lpred(A_B,p_B)<0$.
Consequently, the original isotropic-target regularizer is larger than its baseline value, and its weighted increase exceeds $-\Delta\Lpred$.
\end{proof}

\subsection{Spectral bounds for admissible targets}

\begin{corollary}[Spectral bounds for admissible targets]
\label{cor:boundaries}
Let $D\ge2$ and $\kappa\ge1$.
Every target $\Lambda=\diag(v_1,\ldots,v_D)\in\mathcal T_{D,\kappa}$ satisfies
\begin{equation}
 \begin{aligned}
 \frac{D}{1+(D-1)\kappa}
 &\le v_i\le\frac{D\kappa}{\kappa+D-1},\\
 d_{\mathrm{eff}}(\Lambda)
 &:=\frac{(\tr\Lambda)^2}{\tr(\Lambda^2)}
 \ge\frac{4\kappa}{(\kappa+1)^2}D.
 \end{aligned}
 \label{eq:dim}
\end{equation}
\end{corollary}

A fixed condition-number bound gives a lower bound on target effective dimension; at $\kappa=2$, the bound is $8D/9$.
This bound quantifies variance concentration in the target covariance; the covariance of a trained neural representation also depends on the encoder.
In two dimensions, trace two and the variance floor $v_i\ge2/(\kappa+1)$ describe the same feasible spectra as the condition-number constraint.

\begin{proof}
Write $v_{\min}=\min_i v_i$ and $v_{\max}=\max_i v_i$.
The trace constraint and $v_{\max}\le\kappa v_{\min}$ imply
\[
 \begin{aligned}
 D=\sum_i v_i&\le v_{\min}+(D-1)\kappa v_{\min},\\
 D=\sum_i v_i&\ge v_{\max}+(D-1)v_{\max}/\kappa.
 \end{aligned}
\]
Thus $v_{\min}\ge D/[1+(D-1)\kappa]$ and $v_{\max}\le D\kappa/(\kappa+D-1)$, giving the eigenvalue bounds.
Also, $v_i\in[v_{\min},\kappa v_{\min}]$ implies
\[
 (v_i-v_{\min})(\kappa v_{\min}-v_i)\ge0,
 \qquad
 v_i^2\le(\kappa+1)v_{\min}v_i-\kappa v_{\min}^2.
\]
Summing and completing the square yields
\[
 \sum_i v_i^2
 \le D\bigl[(\kappa+1)v_{\min}-\kappa v_{\min}^2\bigr]
 \le\frac{D(\kappa+1)^2}{4\kappa}.
\]
Since $\sum_i v_i=D$, the effective-dimension bound follows.
For $\kappa>1$, equality requires all eigenvalues to be in $\{v_{\min},\kappa v_{\min}\}$ and $v_{\min}=(\kappa+1)/(2\kappa)$.
The trace then requires exactly $D/(\kappa+1)$ upper eigenvalues.
Thus equality holds precisely when this number is an integer and the spectrum has that multiplicity.
For $D=2$ and trace two, the eigenvalues are $v_{\min}$ and $2-v_{\min}$, so
\[
 \cond(\Lambda)\le\kappa
 \quad\Longleftrightarrow\quad
 \frac{2-v_{\min}}{v_{\min}}\le\kappa
 \quad\Longleftrightarrow\quad
 v_{\min}\ge\frac{2}{\kappa+1}.
\]
\end{proof}

\phantomsection
\label{app:target_parameterization}
The logit parameterization in Section~\ref{sec:joint_training} covers every admissible target.
For any $\Lambda=\diag(v_1,\ldots,v_D)\in\mathcal T_{D,\kappa}$, let $v_{\min}=\min_i v_i$ and $v_{\max}=\max_i v_i$.
Choose $\alpha_i=\log v_i-(\log v_{\max}+\log v_{\min})/2$.
Since $\log v_{\max}-\log v_{\min}\le\log\kappa$, these logits lie in $[-\tfrac12\log\kappa,\tfrac12\log\kappa]$.
The common shift cancels in the normalized exponential, and $\sum_i v_i=D$ gives $D e^{\alpha_i}/\sum_j e^{\alpha_j}=v_i$.

\begin{samepage}
\subsection{Scope of the analysis}
\label{app:scope}

Theorem~\ref{thm:obstruction} holds for any $\rho_1,\rho_2>0$.
Theorem~\ref{thm:bounded} assumes $\rho_1<\rho_2$: minimizing $\tr(LR_0)$ then assigns the larger target variance to the direction with larger state variance.
For $\rho_1>\rho_2$, the trace over the disk in~\eqref{eq:disk} is minimized at $s=-b,t=0$.
The same compactness and localization argument gives
\[
 K_{\eta,T}\longrightarrow q\diag(1-b,1+b),\qquad
 u_T\longrightarrow\frac{\delta+b}{1+\delta b}
 =\frac{c\kappa-1}{c\kappa+1}.
\]
Since
\[
 \frac{\delta+b}{1+\delta b}-\delta
 =\frac{b(1-\delta^2)}{1+\delta b}>0,
\]
every $\kappa>1$ gives strictly higher limiting regret than the baseline.
If $\rho_1=\rho_2$, all target orientations tie in the leading trace term.
\end{samepage}

\paragraph{Two-dimensional target spectrum.}
In two dimensions, the target spectrum is fully determined by $\kappa$ up to permutation. 
In higher dimensions, Proposition~4 additionally selects the multiplicity of the larger target variance as a function of the predictive structure.

For the same control family with diagonal task metric $Q=\diag(\gamma_1,\gamma_2)\succ0$, set $\kappa_Q=c\gamma_1/\gamma_2$.
Under $\rho_1<\rho_2$, if $\kappa_Q>1$, choosing $\kappa=\kappa_Q$ in~\eqref{eq:targetlimit} gives
\[
 M_{\eta,T}\longrightarrow\frac{2q}{\tr(\Sigma Q)}Q.
\]
Indeed, the limiting target allocation is $2\Sigma Q/\tr(\Sigma Q)$.
The planner therefore attains zero limiting regret for the $Q$-weighted task cost.
The case $Q=I$ recovers $\kappa_Q=c$.
This calibration uses the state covariance and task metric.
Under a physical-state coordinate change $\bar x=Tx$, the same physical cost has metric $T^{-\top}QT^{-1}$.

The asymptotic comparisons above hold the regularization weights fixed and strictly positive as $\eta \downarrow 0$.
Allowing the regularization weights to scale with $\eta$ defines a different asymptotic regime.
\section{Prediction-Driven Target Spectrum Selection}
\label{app:adaptive_spectrum}

This appendix characterizes the target spectrum selected by the joint prediction--regularization objective.
We first establish its structure in arbitrary dimension and identify the role of the state-whitened noise covariance.
We then give two three-dimensional training distributions that select different variance allocations under the same trace and condition-number constraints.
As in Appendix~\ref{app:targetfamily}, $A$ and $p$ denote the encoder and predictor in the original latent coordinates.
We use the prediction-loss normalization $\tfrac12\E\norm{p(Ax,u)-Ax'}^2$ throughout.
For dimension $D$, multiplying the coordinate-averaged implementation objective by $D/2$ gives this convention with a correspondingly rescaled positive regularization weight; this does not change the global minimizers or planning costs.

\subsection{Spectrum selection in arbitrary dimension}
\label{app:general_spectrum}

Target learning selects both the variance values and the number of directions assigned to each value.
We characterize this selection in arbitrary dimension before applying it to the two control models.
For $D\ge2$, $\kappa>1$, and a positive-definite state-whitened noise covariance $R$, define
\begin{equation}
 \begin{aligned}
 \mathcal S_{D,\kappa}
 &=\{L\succ0:\tr L=D,\ \cond(L)\le\kappa\},\\
 \mathcal G_\eta(K,L)
 &=\tfrac\eta2\tr(KR)
   +\lambda\Reg(L^{-1/2}KL^{-1/2}),
 \qquad K\succeq0,\quad L\in\mathcal S_{D,\kappa}.
 \end{aligned}
 \label{eq:general_reduction}
\end{equation}
Let $0<r_1\le\cdots\le r_D$ be the eigenvalues of $R$, and set
\begin{equation}
 a_m=\frac{D}{D+m(\kappa-1)},\qquad
 \phi_m=a_m\left[\tr R+(\kappa-1)\sum_{i=1}^m r_i\right],
 \quad 1\le m<D.
\label{eq:rank_selector}
\end{equation}
Here $a_m$ and $\kappa a_m$ are the lower and upper target variances, and $\phi_m$ is the trace cost $\tr(LR)$ when the upper variance is assigned to the $m$ lowest-noise directions.

\begin{proposition}[Structure of optimal target spectra]
\label{prop:target_spectrum}
Assume the statistic conditions of Lemma~\ref{lem:finite}, and fix $D\ge2$, $R\succ0$, $\kappa>1$, and $\lambda>0$.
For all sufficiently small $\eta>0$, every global minimizer of \eqref{eq:general_reduction} has a target geometry of the form
\[
 L_\eta=a_m\bigl[I+(\kappa-1)P\bigr],
 \qquad P=P^\top=P^2,\quad \operatorname{rank}P=m,
 \quad 1\le m<D.
\]
Its spectrum has two values, $\kappa a_m$ and $a_m$, and its condition number equals $\kappa$.
Moreover, $L_\eta^{-1/2}K_\eta L_\eta^{-1/2}\to qI$ uniformly over the global minimizers as $\eta\downarrow0$.

For every sequence $\eta_n\downarrow0$ and choice of global minimizers $(K_n,L_n)$, each accumulation point of $L_n$ minimizes $\tr(LR)$ on $\mathcal S_{D,\kappa}$.
Its multiplicity $m$ minimizes \eqref{eq:rank_selector}, and its larger-variance eigenspace is a span of $m$ lowest-noise eigendirections of $R$.
If $\phi_m$ has a unique minimizing index $m_*$ and $r_{m_*}<r_{m_*+1}$, the target geometries converge to the unique limit in any matrix norm, uniformly over the global minimizers.
\end{proposition}

\begin{proof}
The set $\mathcal S_{D,\kappa}$ is compact by the eigenvalue bounds in~\eqref{eq:dim}.
It is convex because $\lambda_{\max}(L)-\kappa\lambda_{\min}(L)$ is a convex function of the symmetric matrix $L$.
For each fixed $\eta>0$, the joint objective is coercive in $K$ uniformly over $L$.
Compactness of the target set therefore gives a global minimum.
For an effective noise covariance $Y\succ0$, define the value obtained after optimizing the standardized feature covariance:
\[
 \Phi_\eta(Y)=\min_{C\succeq0}
 \left\{\lambda\Reg(C)+\tfrac\eta2\tr(CY)\right\}.
\]
The minimum exists by continuity and coercivity.
Substituting $K=L^{1/2}CL^{1/2}$ in~\eqref{eq:general_reduction} gives the coefficient $L^{1/2}RL^{1/2}$ in the trace.
This matrix and $R^{1/2}LR^{1/2}$ have the same eigenvalues.
Orthogonal invariance of $\Reg$ therefore gives
\begin{equation}
 \min_{K\succeq0}\mathcal G_\eta(K,L)
 =\Phi_\eta(R^{1/2}LR^{1/2}).
 \label{eq:target_envelope}
\end{equation}

The matrices $Y=R^{1/2}LR^{1/2}$ range over a compact convex set with a uniform positive lower eigenvalue bound.
Comparing any minimizer $C$ in the definition of $\Phi_\eta(Y)$ with $qI$ gives
\[
 \tfrac\eta2\tr(CY)
 +\lambda[\Reg(C)-\Reg(qI)]
 \le\tfrac{\eta q}{2}\tr Y.
\]
Both left-hand terms are nonnegative.
This bounds $C$ uniformly and makes its regularizer gap tend uniformly to zero.
By compactness and Lemma~\ref{lem:finite}, all such minimizers approach $qI$ uniformly and are positive definite for sufficiently small $\eta>0$.
Orthogonal conjugacy transfers this convergence to the standardized covariance $L^{-1/2}KL^{-1/2}$ in the original reduced objective.

As an infimum of affine functions of $Y$, $\Phi_\eta$ is concave.
It is strictly concave on the displayed set for small positive $\eta$.
Indeed, equality in the concavity inequality for distinct $Y_0,Y_1$ would force a minimizer at their proper convex combination to minimize both endpoint objectives.
This minimizer is interior, so its two stationarity equations give
\[
 \lambda\nabla_C\Reg(C)+\tfrac\eta2Y_i=0,
 \qquad i=0,1,
\]
contradicting $Y_0\ne Y_1$.
The map $L\mapsto R^{1/2}LR^{1/2}$ is linear and injective.
Hence~\eqref{eq:target_envelope} is strictly concave in $L$, and every minimizing $L$ is an extreme point of $\mathcal S_{D,\kappa}$.

We now characterize those extreme points.
If $\cond(L)<\kappa$, small traceless symmetric perturbations of both signs remain feasible, so $L$ is not extreme.
Suppose $\cond(L)=\kappa$, with minimum eigenvalue $a$.
If any eigenvalue lies strictly between $a$ and $\kappa a$, perturb all upper endpoint eigenvalues by $\kappa\epsilon$ and all lower endpoint eigenvalues by $\epsilon$.
Distribute the opposite total perturbation among the intermediate eigenvalues to preserve the trace.
Both signs of sufficiently small $\epsilon$ remain feasible.
Such an $L$ is again nonextreme.

The remaining matrices have the stated form $L=a_m[I+(\kappa-1)P]$.
To show that each is extreme, suppose $L=tX+(1-t)Y$ with $0<t<1$ and $X,Y\in\mathcal S_{D,\kappa}$.
For unit vectors $e_+\in\operatorname{ran}P$ and $e_-\in\ker P$,
\[
 e_+^\top X e_+\le\lambda_{\max}(X)
 \le\kappa\lambda_{\min}(X)\le\kappa e_-^\top X e_-,
\]
and the same inequalities hold for $Y$.
Since $e_+^\top L e_+=\kappa e_-^\top L e_-$, equality holds throughout both chains.
Thus the two subspaces are respectively maximal and minimal eigenspaces of $X$ and $Y$.
The trace condition forces $X=Y=L$.

For any trace minimizer $L_*$, comparison of a global minimizer $(K_\eta,L_\eta)$ with $(qL_*,L_*)$ gives $\tr(K_\eta R)\le q\tr(L_*R)$.
Along any convergent subsequence $L_\eta\to\bar L$, the uniform covariance limit gives $K_\eta\to q\bar L$.
Hence $\bar L$ minimizes $\tr(LR)$.
The extreme-point set is a finite union of compact projection orbits, so $\bar L$ has the same two-value spectrum.
The rearrangement inequality places the larger eigenvalues along the $m$ smallest eigenvalues of $R$, giving $\phi_m$.
A unique minimizing index and the corresponding eigengap make the spectral projector, and therefore $\bar L$, unique.
\end{proof}

For a square linear encoder with training-state covariance $\Sigma\succ0$, write $B_\eta=A_\eta\Sigma^{1/2}=O_\eta K_\eta^{1/2}$ and $L_\eta=O_\eta^\top\Lambda_\eta O_\eta$, where $O_\eta$ is orthogonal.
This rotation transfers the spectrum of $L_\eta$ to the diagonal output target $\Lambda_\eta$.
At the global optima in the proposition, $K_\eta-qL_\eta\to0$, so the feature covariance satisfies $A_\eta\Sigma A_\eta^\top-q\Lambda_\eta\to0$.
Its finite-noise spectrum may have more than two distinct eigenvalues.
The selector $\phi_m$ depends on the relative magnitudes of the noise eigenvalues.

\paragraph{Selection among tied spectra.}
For fixed $D$, $R\succ0$, $\lambda>0$, and $\kappa>1$, uniform localization and the positive-definite Hessian give a common smooth minimizer branch in $\eta Y$ by the implicit-function theorem.
Expanding the optimum value gives, uniformly for $L\in\mathcal S_{D,\kappa}$,
\begin{equation}
 \begin{aligned}
 \Phi_\eta(Y)
 &=\lambda\Reg(qI)+\frac{\eta q}{2}\tr Y\\
 &\quad-\frac{\eta^2D}{16\lambda c_N}
 \left[(D+2)\tr(Y^2)-(\tr Y)^2\right]+O(\eta^3),
 \end{aligned}
 \label{eq:second_order_envelope}
\end{equation}
where $Y=R^{1/2}LR^{1/2}$ and $c_N=r_N''(q)>0$.

Write $\phi(L)=\tr(LR)$ and $\psi(L)=\tr((LR)^2)=\tr(Y^2)$.
Every target limit $\bar L$ minimizes $\phi$.
Compare a convergent sequence of global minimizers with any minimizer $L_*$ of $\phi$ and divide the value inequality by $\eta^2$.
The nonnegative first-order gap and uniform remainder give $\psi(\bar L)\ge \psi(L_*)$.
Thus its two-level rank maximizes
\begin{equation}
 \psi_m=a_m^2\left[\kappa^2\sum_{i=1}^m r_i^2
                  +\sum_{i=m+1}^D r_i^2\right]
 \label{eq:second_order_selector}
\end{equation}
among indices tied for the smallest $\phi_m$.
A unique maximizing index $m$ and $r_m<r_{m+1}$ give a unique target limit.
Remaining ties in $\psi_m$ depend on higher-order terms.
For $D=3$, $\kappa=2$, and $R=\diag(1,2,4)$, $\phi_1=\phi_2=6$ and $\psi_1=27/2>324/25=\psi_2$, so every optimum at sufficiently small positive noise uses the $m=1$ branch.

\subsection{Two training distributions with different selected spectra}
\label{app:two_spectrum_models}

Let $j\in\{1,2\}$ index the models.
Define
\begin{equation}
 \begin{aligned}
 R_1&=\diag(1,10,11),&
 \Sigma_1&=\diag(3/2,3/4,3/4),\\
 R_2&=\diag(1,2,10),&
 \Sigma_2&=\diag(6/5,6/5,3/5),\\
 g_1&=(1,1,0)^\top,& d_1&=(1,-1,0)^\top,\\
 g_2&=(0,1,1)^\top,& d_2&=(0,1,-1)^\top.
 \end{aligned}
 \label{eq:two_models}
\end{equation}
Fix $\tau^2=1/10$ and set
\[
 W_{j,\eta}=\eta\Sigma_j^{1/2}R_j\Sigma_j^{1/2},\qquad
 F_{j,\eta}=(\Sigma_j-\tau^2g_jg_j^\top-W_{j,\eta})^{1/2}
             \Sigma_j^{-1/2}.
\]
Since $g_1^\top\Sigma_1^{-1}g_1=2$ and $g_2^\top\Sigma_2^{-1}g_2=5/2$, the matrix inside the square root is positive definite for sufficiently small $\eta>0$.
Thus $F_{j,\eta}$ is invertible.
Draw independent reset tuples
\[
 x\sim\mathcal N(0,\Sigma_j),\quad
 u\sim\mathcal N(0,\tau^2),\quad
 \xi\sim\mathcal N(0,W_{j,\eta}),\qquad
 x'=F_{j,\eta}x+g_ju+\xi.
\]
The marginal covariance of $x'$ is also $\Sigma_j$.
We use independent tuples within each batch and the same finite-statistic assumptions as Lemma~\ref{lem:finite}, which applies in dimension three.

We optimize an unrestricted encoder $A\in\R^{3\times3}$ and a predictor $p$ that is linear in $(Ax,u)$.
For any fixed $\lambda>0$, their joint training objective with the target is
\begin{equation}
 \mathcal L_{j,\eta}(A,p,\Lambda)
 =\tfrac12\E\norm{p(Ax,u)-Ax'}^2
  +\lambda\Reg\bigl(\Lambda^{-1/2}A\Sigma_jA^\top
                         \Lambda^{-1/2}\bigr),
 \quad \Lambda\in\mathcal T_{3,2}.
 \label{eq:three_dimensional_objective}
\end{equation}

For evaluation, use $x_0=F_{j,\eta}^{-1}d_j$, goal $x_g=0$, and $u\in[-1,1]$.
Both physical costs satisfy
\begin{equation}
 J_{*,j}(u)=\E\norm{d_j+g_ju+\xi}^2
          =2+2u^2+\tr W_{j,\eta},\qquad
 \Regret_j(u)=2u^2.
 \label{eq:three_dimensional_task}
\end{equation}
The task-optimal action is zero; the task cost is used only for evaluation.

\begin{proposition}[Distribution-dependent selected spectra]
\label{prop:adaptive_spectrum}
For the two three-dimensional models above, fix $\kappa=2$ and strictly positive regularization weights, held constant as $\eta\downarrow0$.
At global training optima, the target geometries in state-whitened coordinates converge to
\[
 L_1^*=\Sigma_1=\diag(3/2,3/4,3/4),\qquad
 L_2^*=\Sigma_2=\diag(6/5,6/5,3/5).
\]
The limiting target spectrum therefore assigns the larger variance to one direction in the first model and two directions in the second, despite the identical ordering of their noise eigenvalues.
Every global optimum is noncollapsed for sufficiently small positive noise, and its state-space metric converges to $qI$.
The exact Euclidean planner consequently has zero limiting task regret in each model.
These limits hold uniformly over global minimizers as $\eta\downarrow0$.
\end{proposition}

\begin{proof}
\textbf{Selected target geometries. }

Write $B=A\Sigma_j^{1/2}=OK^{1/2}$ and $L=O^\top\Lambda O$.
Appendix~\ref{app:bounded} gives
\begin{equation}
 \begin{aligned}
 \mathcal G_{j,\eta}(K,L)
 &=\tfrac\eta2\tr(KR_j)
   +\lambda\Reg(L^{-1/2}KL^{-1/2}),\\
 &K\succeq0,\qquad
 L\in\mathcal S:=\{L\succ0:\tr L=3,\ \cond(L)\le2\}.
 \end{aligned}
 \label{eq:three_dimensional_reduction}
\end{equation}
For $K\succ0$, the polar construction realizes each pair with a diagonal output target, encoder orientation, and $p_A(z,u)=AF_{j,\eta}A^{-1}z+Ag_ju$.
The target domain $\mathcal S$ is compact, with eigenvalues uniformly bounded away from zero by~\eqref{eq:dim}.

Let $L_*$ be the unique minimizer of $\tr(LR_j)$ over $\mathcal S$, identified below.
For each $\eta>0$, continuity and coercivity give a reduced global minimum.
The comparison argument in the proof of Proposition~\ref{prop:target_spectrum} gives
\begin{equation}
 \tfrac\eta2\tr(KR_j)
 +\lambda[\Reg(C)-\Reg(qI)]
 \le\tfrac{\eta q}{2}\tr(L_*R_j),\qquad
 C=L^{-1/2}KL^{-1/2}.
 \label{eq:spectrum_comparison}
\end{equation}
Since $R_j\succeq I$, the trace bound controls $K$.
The uniform lower eigenvalue bound on $L$ also bounds $C$.
The regularizer gap is $O(\eta)$, so compactness and Lemma~\ref{lem:finite} give $C\to qI$ uniformly.
Hence $K=L^{1/2}CL^{1/2}\to q\bar L$ along every subsequence $L\to\bar L$.
Equation~\eqref{eq:spectrum_comparison} and uniqueness force $\bar L=L_*$.
Compactness then gives convergence uniformly over all reduced minimizers:
\begin{equation}
 L\longrightarrow L_*,\qquad K\longrightarrow qL_*.
 \label{eq:spectrum_selection}
\end{equation}
Thus $K\succ0$ for every reduced minimizer at sufficiently small positive noise.
Each positive-definite reduced minimizer is realizable by an encoder, predictor, and admissible target.
Thus the full and reduced objectives have the same minimum value.
Every full minimizer induces a reduced minimizer and attains the conditional-mean prediction bound.
Its feature covariance shares $K$'s eigenvalues and is noncollapsed.

We minimize $\tr(LR_j)$ over both the target spectrum and its orientation.
Each $R_j$ has strictly increasing diagonal entries.
For fixed eigenvalues of $L$, the trace is minimized by aligning its largest eigenvalue with the smallest noise entry.
To see this, expand $\tr(\diag(r)V\diag(v)V^\top)$ as $\sum_{i,k}r_iv_kV_{ik}^2$.
The squared entries form a doubly stochastic matrix.
Minimizing the linear expression over such matrices admits a permutation minimizer, and the rearrangement inequality orders the eigenvalues oppositely.
Strict ordering of the noise entries makes the minimizing matrix unique.
Rotations within a repeated target eigenspace leave that matrix unchanged.

It remains to optimize $v_1\ge v_2\ge v_3$ subject to $\sum_i v_i=3$ and $v_1\le2v_3$.
This feasible region is the triangle with vertices $(1,1,1)$, $(3/2,3/4,3/4)$, and $(6/5,6/5,3/5)$.
\par\noindent
\begin{minipage}{\linewidth}
The corresponding trace values are
\begin{center}
\begin{tabular}{lrr}
\toprule
Spectrum & $\tr(LR_1)$ & $\tr(LR_2)$ \\
\midrule
$(1,1,1)$ & $22$ & $13$ \\
$(3/2,3/4,3/4)$ & $69/4$ & $21/2$ \\
$(6/5,6/5,3/5)$ & $99/5$ & $48/5$ \\
\bottomrule
\end{tabular}
\end{center}
\end{minipage}
\par Each column has a unique minimizing vertex.
Thus $L_{1,\mathrm{learn}}^*=\Sigma_1$ and $L_{2,\mathrm{learn}}^*=\Sigma_2$.
Although the two noise spectra have the same ordering, their relative magnitudes select different numbers of larger target eigenvalues.

\begin{samepage}
\paragraph{Planning regret.}

Since $L_{j,\mathrm{learn}}^*=\Sigma_j$,
\[
 A^\top A=\Sigma_j^{-1/2}K\Sigma_j^{-1/2}\longrightarrow qI.
\]
The Euclidean planner minimizes $(d_j+g_ju)^\top A^\top A(d_j+g_ju)$.
Its unconstrained action converges to zero and hence is feasible for sufficiently small noise.
The limiting task regret is zero in both models.
\end{samepage}

\end{proof}

The construction sets $\Sigma_j=L_{j,\mathrm{learn}}^*$, yielding the full limiting state metric $qI$.
In general, agreement with Euclidean state costs over all residuals requires $L_*$ proportional to $\Sigma$; under Proposition~\ref{prop:target_spectrum}, the state covariance must therefore have exactly two distinct eigenvalues with ratio $\kappa$, up to an overall scale.
Agreement on a restricted feasible residual set can hold under weaker conditions.

\section{General-Dimensional Geometry, Finite-Horizon Planning, and Nonlinear Stability}
\label{ag:appendix}

This appendix proves the arbitrary-dimensional joint-optimum result used in Section~\ref{sec:joint_geometry}, gives the explicit finite-horizon construction underlying Theorem~\ref{thm:finite_horizon_separation}, and records the corresponding learned-target and nonlinear extensions.
Throughout, $\norm{\cdot}$ denotes the Euclidean norm for vectors and the operator norm for matrices.
Matrix square roots are symmetric positive-semidefinite roots.
All limits fix the dimension, horizon, batch size, quadrature, action statistics, and positive regularization weights.

\subsection{Standing regularizer assumptions}
\label{ag:regularizer}

We use the expected finite-batch statistic $\Reg$ of Appendix~\ref{app:finite} in dimension $D$.
Specifically, $N>2$, the quadrature has finitely many nonnegative knots $t_k$ and weights $w_k$, at least one $w_kt_k$ is positive, projection directions are uniform on $\Sph^{D-1}$, and $r_N'(0)<0$ in Eq.~\eqref{eq:r}.
Lemma~\ref{lem:finite} and its proof show that $\Reg$ is continuous, nonnegative, and invariant under orthogonal conjugation on the positive-semidefinite cone.
Its unique minimum is $qI_D$, where $q\in(0,1)$, and it is smooth near this minimum.
With $c_N=r_N''(q)>0$, Eq.~\eqref{eq:hessian} gives
\[
 D^2\Reg(qI_D)[E,E]
 =\frac{c_N}{D(D+2)}\left[(\tr E)^2+2\tr(E^2)\right]
\]
for every symmetric $E$.
In particular, the Hessian is positive definite on symmetric matrices.
These are the regularizer properties used below; the finite-statistic proof need not be repeated.

\subsection{General Gaussian training model and isotropic joint optimum}
\label{ag:training}

For the proof of Proposition~\ref{prop:isotropic_joint_metric}, consider the Gaussian training model
\begin{equation}
 \begin{aligned}
 x'&=F_\eta x+Ga+\xi,\\
 x&\sim\mathcal N(0,\Sigma),\qquad
 a\sim\mathcal N(0,\Gamma),\qquad
 \xi\sim\mathcal N(0,\eta W),
 \end{aligned}
 \label{ag:training-model}
\end{equation}
where $x$, $a$, and $\xi$ are mutually independent, $\Sigma,W\in\mathbb R^{D\times D}$ and $\Gamma\in\mathbb R^{d_a\times d_a}$ are positive definite, and $G\in\mathbb R^{D\times d_a}$.
Encoders $A\in\mathbb R^{D\times D}$ may be singular; predictors are linear in $(Ax,a)$.
Write
\[
 \Lpred(A,p)=\tfrac12\mathbb E\norm{p(Ax,a)-Ax'}^2,
 \qquad R=\Sigma^{-1/2}W\Sigma^{-1/2}.
\]
Multiplying a coordinate-averaged prediction objective by $D/2$ gives this normalization, with its regularization weight multiplied by the same factor.
If the regularizer is averaged over current and successor features, the stationary-marginal condition
\[
 F_\eta\Sigma F_\eta^\top+G\Gamma G^\top+\eta W=\Sigma
\]
makes the two expected statistics identical.
Independent trajectories also suffice when examples are independent within each regularized time slice.

For reference, write the isotropic objective as
\begin{equation}
 \mathcal L_B(A,p)
 =\Lpred(A,p)+\lambda_B\Reg(A\Sigma A^\top).
 \label{ag:baseline-objective}
\end{equation}

\begin{proof}[Proof of Proposition~\ref{prop:isotropic_joint_metric}]
\emph{Prediction reduction.}
Let $B=A\Sigma^{1/2}$ and $K=B^\top B\succeq0$.
Independence and zero-mean noise give, including for singular $A$,
\begin{equation}
 \Lpred(A,p)
 =\tfrac12\mathbb E\norm{p(Ax,a)-A(F_\eta x+Ga)}^2
   +\tfrac\eta2\tr(KR)
 \ge\tfrac\eta2\tr(KR).
 \label{ag:prediction-reduction}
\end{equation}
Every positive-definite $K$ attains this lower bound with an invertible encoder and
\begin{equation}
 p_A(z,a)=AF_\eta A^{-1}z+AGa.
 \label{ag:exact-predictor}
\end{equation}

\emph{Isotropic target.}
The matrices $BB^\top$ and $B^\top B$ are orthogonally similar.
The reduced objective is therefore
\[
 f_\eta(K)=\tfrac\eta2\tr(KR)+\lambda_B\Reg(K),\qquad K\succeq0.
\]
It is continuous and coercive for every $\eta>0$ since $R\succ0$ and $\Reg\ge0$.
Comparing any minimizer with $qI_D$ yields
\begin{equation}
 \tfrac\eta2\tr(KR)+\lambda_B[\Reg(K)-\Reg(qI_D)]
 \le\tfrac{\eta q}{2}\tr R.
 \label{ag:baseline-comparison}
\end{equation}
Both left-hand terms are nonnegative.
Thus all minimizing $K$ lie in a common compact set and their regularizer gaps tend to zero.
Every accumulation point is $qI_D$.
This localization is uniform: otherwise a sequence of minimizers staying a fixed distance from $qI_D$ would have a subsequence converging to it.
By the Hessian formula and continuity, $\Reg$ is strictly convex on a small convex neighborhood of $qI_D$.
All global minimizers eventually lie there, so the reduced minimizer is unique and positive definite.

It is attained by $A=K^{1/2}\Sigma^{-1/2}$ and Eq.~\eqref{ag:exact-predictor}.
A global minimizer of the full objective must attain both the reduced minimum and equality in Eq.~\eqref{ag:prediction-reduction}.
Its encoder is invertible; since $\operatorname{Cov}(Ax,a)=\diag(A\Sigma A^\top,\Gamma)$ is positive definite, equality in prediction identifies the linear predictor coefficients.
Finally $A^\top A=\Sigma^{-1/2}K\Sigma^{-1/2}$ gives Eq.~\eqref{eq:general_baseline_metric}, and the prediction-loss identity follows from Eq.~\eqref{ag:prediction-reduction}.
\end{proof}

\subsection{Proof of Proposition~\ref{prop:learned_target_metric}}
\label{ag:target-geometry}

In addition to the diagonal target family $\mathcal T_{D,\kappa}$ in Eq.~\eqref{eq:family}, define
\[
 \mathcal S_{D,\kappa}
 =\{L\succ0:\tr L=D,\ \cond(L)\le\kappa\}.
\]
The set $\mathcal S_{D,\kappa}$ incorporates the encoder's free orientation; it does not introduce a full-covariance target parameterization.

\begin{proof}
Let $B=A\Sigma^{1/2}$ and write the polar decomposition as
$B=OK^{1/2}$, using an orthogonal extension if $B$ is singular.
Set $L=O^\top\Lambda O$.
Then
\[
 \Lambda^{-1/2}A\Sigma A^\top\Lambda^{-1/2}
 =O(L^{-1/2}KL^{-1/2})O^\top.
\]
Together with the prediction reduction in Eq.~\eqref{ag:prediction-reduction} and orthogonal invariance of $\Reg$, the learned-target objective reduces to
\[
 g_\eta(K,L)
 =\tfrac\eta2\tr(KR)
 +\lambda_T\Reg(L^{-1/2}KL^{-1/2}),
 \qquad
 K\succeq0,\quad L\in\mathcal S_{D,\kappa}.
\]

Every admissible pair $(K,L)$ with $K\succ0$ is realizable by the full model:
choose $O$ such that $OLO^\top$ is diagonal, set
$\Lambda=OLO^\top$ and $A=OK^{1/2}\Sigma^{-1/2}$, and use the predictor in Eq.~\eqref{ag:exact-predictor}.
For $L\in\mathcal S_{D,\kappa}$,
\[
 \lambda_{\min}(L)\ge \frac{D}{1+(D-1)\kappa},
 \qquad
 \lambda_{\max}(L)\le D,
\]
so $\mathcal S_{D,\kappa}$ is compact.
Since the reduced objective is continuous and coercive in $K$, it attains a minimum for every $\eta>0$.

Let $L_0$ minimize $\tr(LR)$ on $\mathcal S_{D,\kappa}$.
For any reduced global minimizer $(K,L)$, comparison with $(qL_0,L_0)$ gives, with
$C=L^{-1/2}KL^{-1/2}$,
\begin{equation}
 \tfrac\eta2\tr(KR)
 +\lambda_T[\Reg(C)-\Reg(qI_D)]
 \le\tfrac{\eta q}{2}\tr(L_0R).
 \label{ag:target-comparison}
\end{equation}
The two terms on the left are nonnegative.
The comparison therefore bounds $K$ uniformly over minimizing pairs and forces the regularizer gap to vanish as $\eta\downarrow0$.
Compactness of the minimizing sets and the uniqueness of the covariance minimizer of $\Reg$ then give, uniformly over reduced global minimizers,
\begin{equation}
 L^{-1/2}KL^{-1/2}\longrightarrow qI_D.
 \label{ag:standardized-target-limit}
\end{equation}
Consequently,
\[
 K-qL
 =L^{1/2}\!\left(L^{-1/2}KL^{-1/2}-qI_D\right)L^{1/2}
 \longrightarrow0
\]
uniformly.
In particular, every minimizing $K$ is positive definite for all sufficiently small $\eta$.
The reduced minimum is therefore attained by the full objective, proving existence of a global minimum in Eq.~\eqref{eq:general_target_objective}.
Moreover, every full global minimizer must attain both the reduced minimum and equality in the prediction bound, so its encoder is invertible and its predictor is the exact encoded conditional mean in Eq.~\eqref{ag:exact-predictor}.

Along any sequence of global minimizers with $L\to\bar L$,
Eq.~\eqref{ag:target-comparison} and $K-qL\to0$ imply
\[
 \tr(\bar L R)\le \tr(L_0R).
\]
Hence every accumulation point $\bar L$ minimizes $\tr(LR)$ on $\mathcal S_{D,\kappa}$.
If this trace minimizer is unique, compactness gives uniform convergence
$L\to L_0$ over global minimizers.
Finally,
\[
 A^\top A
 =\Sigma^{-1/2}K\Sigma^{-1/2},
\]
so conjugating $K-qL\to0$ by $\Sigma^{-1/2}$ yields the metric limit in
Eq.~\eqref{eq:general_target_metric}.
The trace bounds in Eq.~\eqref{ag:target-comparison} also imply
$\Lpred=\eta\tr(KR)/2\to0$.
The same bounds give uniform upper and lower bounds on the encoder singular values for sufficiently small noise.
This proves Proposition~\ref{prop:learned_target_metric}.
\end{proof}

The task-alignment condition used in Section~\ref{sec:method_geometry} follows directly from the limiting metric in Eq.~\eqref{eq:general_target_metric}.
For a task metric $Q\succ0$, proportionality to $Q$ together with $\tr L=D$ gives
\[
 L_*=\frac{D\Sigma^{1/2}Q\Sigma^{1/2}}{\tr(\Sigma Q)},
\]
which is Eq.~\eqref{eq:target_alignment}.
Here $L=O^\top\Lambda O$ incorporates the encoder orientation.
The diagonal target alone therefore does not determine the state metric: the aligned geometry must be both feasible and selected by predictive training.

\subsection{Exact rollouts and finite-horizon regret}
\label{ag:rollouts}

For a deterministic planned sequence $U=(a_0,\ldots,a_{H-1})$, use independent innovations in Eq.~\eqref{ag:training-model} and a fixed initial state $x_0$.
Set
\begin{align}
 \mu_{\eta,H}(U)
 &=F_\eta^Hx_0+\sum_{t=0}^{H-1}F_\eta^{H-1-t}Ga_t,
 \label{ag:terminal-mean}\\
 \Omega_{\eta,H}
 &=\eta\sum_{j=0}^{H-1}F_\eta^jW(F_\eta^j)^\top.
 \label{ag:terminal-covariance}
\end{align}
For a fixed goal $y$ and $Q\succ0$, let $J_{*,\eta,H}(U)=\mathbb E\norm{x_H-y}_Q^2$, where $\norm{v}_Q^2=v^\top Qv$.
All planning comparisons below use the same feasible sequence set.

\begin{lemma}[Exact finite-horizon encoded means]
\label{ag:exact-rollouts}
At every global optimum covered by Proposition~\ref{prop:isotropic_joint_metric} or Proposition~\ref{prop:learned_target_metric}, autoregressive prediction initialized at $Ax_0$ satisfies
\begin{equation}
 \hat z_H(U)=A\mu_{\eta,H}(U)=\mathbb E[Ax_H\mid x_0,U].
 \label{ag:encoded-mean}
\end{equation}
Moreover,
\begin{align}
 \mathbb E\norm{\hat z_H(U)-Ax_H}^2
 &=\tr(A\Omega_{\eta,H}A^\top),
 \label{ag:rollout-mse}\\
 J_{*,\eta,H}(U)
 &=\norm{\mu_{\eta,H}(U)-y}_Q^2+\tr(Q\Omega_{\eta,H}).
 \label{ag:task-cost}
\end{align}
If $F_\eta\to F_0$ and $H$ is fixed, Eq.~\eqref{ag:rollout-mse} is $O(\eta)$ uniformly over all training optima and deterministic planned sequences.
\end{lemma}

\begin{proof}
Induction with Eq.~\eqref{ag:exact-predictor} gives Eq.~\eqref{ag:encoded-mean}.
The independent, zero-mean innovations give Eq.~\eqref{ag:terminal-covariance}; expanding squared error gives the two remaining identities.
A finite sum of bounded powers of $F_\eta$, together with the uniform bound on $\norm{A}$, gives the $O(\eta)$ estimate.
The task-noise term is independent of the planned sequence and cancels in open-loop regret.
\end{proof}

\begin{lemma}[Convergence of constrained planning regret]
\label{ag:regret-convergence}
Let $\mathcal U\subset\mathbb R^{d_aH}$ be nonempty, compact, and convex, let $F_\eta\to F_0$, and suppose the learned metric $M_\eta=A^\top A$ converges to $M_0\succ0$ uniformly over the training optima under consideration.
Define
\[
 \mathcal Y_0=\{\mu_{0,H}(U):U\in\mathcal U\},\qquad
 v_M=\operatorname*{arg\,min}_{v\in\mathcal Y_0}\norm{v-y}_{M_0}^2,
 \qquad
 v_Q=\operatorname*{arg\,min}_{v\in\mathcal Y_0}\norm{v-y}_Q^2.
\]
The two terminal minimizers are unique.
Every exact latent-planning minimizer $U_\eta$ satisfies, uniformly over training and planning minimizers,
\begin{equation}
 J_{*,\eta,H}(U_\eta)-\min_{U\in\mathcal U}J_{*,\eta,H}(U)
 \longrightarrow\norm{v_M-y}_Q^2-\norm{v_Q-y}_Q^2.
 \label{ag:regret-limit}
\end{equation}
The limit is positive if and only if $v_M\ne v_Q$.
\end{lemma}

\begin{proof}
The terminal set is compact and convex, so strict convexity of each quadratic gives a unique minimizing terminal state.
Optimal action sequences themselves need not be unique.
Since $H$ is fixed and $\mathcal U$ is compact, the mean maps and the latent costs converge uniformly.
For any $\eta_n\downarrow0$ and choices of training and planning minimizers, compactness gives a subsequence $U_{\eta_n}\to\bar U$.
Passing the minimizing inequality to the limit shows that $\mu_{0,H}(\bar U)=v_M$.
Uniform convergence of the mean task costs gives convergence of their minima and proves Eq.~\eqref{ag:regret-limit}, since the common noise term cancels.
If convergence were not uniform, a sequence with a fixed nonzero deviation would have the same convergent-subsequence property, a contradiction.
Strict convexity of the task cost on $\mathcal Y_0$ proves the last claim.
\end{proof}

\subsection{Separation in every dimension and at every finite horizon}
\label{ag:finite-horizon}

\begin{longtheorem}[Explicit finite-horizon construction and target compensation]
\label{ag:separation}
Fix $D\ge2$, a finite $H\ge1$, $\alpha\in(0,1)$, $\bar u>0$, and any nonscalar covariance $\Sigma\succ0$.
Set
\begin{equation}
 F_\eta=(\alpha^2I_D-\eta\Sigma^{-1})^{1/2},
 \qquad \Gamma=(1-\alpha^2)\Sigma,
 \label{ag:transition}
\end{equation}
and train on independent Gaussian tuples
\begin{equation}
 x'=F_\eta x+a+\xi,\qquad
 x\sim\mathcal N(0,\Sigma),\quad
 a\sim\mathcal N(0,\Gamma),\quad
 \xi\sim\mathcal N(0,\eta I_D).
 \label{ag:construction}
\end{equation}
Assume the statistic conditions in Appendix~\ref{ag:regularizer}.
Evaluate from $x_0=0$, with $\norm{a_t}\le\bar u$ at every step, using expected squared Euclidean terminal error to a fixed goal $y$.
Put
\[
 \mathcal U=\{(a_0,\ldots,a_{H-1}):\norm{a_t}\le\bar u\},
 \qquad r_H=\bar u\sum_{j=0}^{H-1}\alpha^j,
\]
and choose $\norm{y}>r_H$ with $\Sigma^{-1}y$ not parallel to $y$.

For any fixed $\lambda_B>0$ and all sufficiently small $\eta>0$, every isotropic-training global optimum is invertible and has exact encoded conditional-mean rollouts.
For every exact $H$-step latent-planning minimizer $U_{B,\eta}$,
\begin{equation}
 \begin{aligned}
 \Regret_{\eta,H}(U_{B,\eta})&\longrightarrow\Delta_H>0,\\
 \Regret_{\eta,H}(U)&:=J_{*,\eta,H}(U)-\min_{V\in\mathcal U}J_{*,\eta,H}(V),
 \end{aligned}
 \label{ag:positive-regret}
\end{equation}
where
\begin{equation}
 \begin{aligned}
 v_B&=(I_D+\nu\Sigma)^{-1}y,\qquad\norm{v_B}=r_H,\qquad\nu>0,\\
 \Delta_H&=\norm{v_B-y}^2-(\norm{y}-r_H)^2.
 \end{aligned}
 \label{ag:explicit-gap}
\end{equation}
The scalar $\nu$ is unique.

Suppose additionally that
\begin{equation}
 \begin{aligned}
 \Sigma&=s[I_D+(c-1)P],\qquad s=\frac{D}{D+r(c-1)},\qquad c>1,\\
 P&=P^\top=P^2,\qquad\rank P=r\in\{1,\ldots,D-1\}.
 \end{aligned}
 \label{ag:two-level}
\end{equation}
For any fixed $\lambda_T>0$ and $\kappa=c$, every learned-target global optimum then has
\begin{equation}
 L\longrightarrow\Sigma,\qquad
 A^\top A\longrightarrow qI_D,\qquad
 \Regret_{\eta,H}(U_{T,\eta})\longrightarrow0.
 \label{ag:compensation}
\end{equation}
All limits are uniform over global training optima and exact planning minimizers.
Both methods have vanishing training prediction loss and fixed-$H$ encoded rollout MSE.
The learned-target regret is strictly smaller for all sufficiently small positive noise.
\end{longtheorem}

\begin{proof}
\emph{Valid training distribution.}
For $0<\eta<\alpha^2\lambda_{\min}(\Sigma)$, $F_\eta$ is positive definite, commutes with $\Sigma$, and satisfies
\[
 F_\eta\Sigma F_\eta^\top=\alpha^2\Sigma-\eta I_D.
\]
Adding $\Gamma$ and $\eta I_D$ proves that $x'$ has covariance $\Sigma$.
Initializing independent training trajectories with this Gaussian marginal and using independent Gaussian actions preserves it at every time slice.
The control matrix is $I_D$, so every state coordinate is actuated.
The Gaussian training policy and the bounded evaluation action set are distinct; the evaluation bound applies equally to both planners.

\emph{Isotropic metric and reachable means.}
Proposition~\ref{prop:isotropic_joint_metric} applies with $R=\Sigma^{-1}$ and gives $M_{B,\eta}\to q\Sigma^{-1}$.
Since $F_\eta\to\alpha I_D$,
\[
 \mu_{0,H}(U)=\sum_{t=0}^{H-1}\alpha^{H-1-t}a_t.
\]
Its image over $\mathcal U$ is exactly the ball $\{v:\norm{v}\le r_H\}$.
The triangle inequality proves one inclusion.
For the reverse inclusion, choose $a_t=v/(\sum_{j=0}^{H-1}\alpha^j)$ for every $t$.
All these actions satisfy the bound.
Every action time has a nonzero coefficient.

\emph{Different optimal terminal states.}
The Euclidean projection of $y$ onto this ball is $v_*=r_Hy/\norm{y}$.
The baseline instead minimizes $(v-y)^\top\Sigma^{-1}(v-y)$ on the ball.
Since $y$ is outside it, the optimum is on the boundary and satisfies
\[
 \Sigma^{-1}(v_B-y)+\nu v_B=0,\qquad\nu>0.
\]
This gives Eq.~\eqref{ag:explicit-gap}.
The norm of $(I_D+\nu\Sigma)^{-1}y$ decreases continuously and strictly from $\norm{y}$ to zero with $\nu$, proving uniqueness.
If $v_B=v_*$, the stationarity equation requires $\Sigma^{-1}y$ to be parallel to $y$, contrary to the assumption.
The task projection is unique, so $\Delta_H>0$.
Such goals exist for every nonscalar $\Sigma$: take a vector with nonzero components in two distinct eigenspaces and scale it beyond radius $r_H$.
Lemma~\ref{ag:regret-convergence} proves Eq.~\eqref{ag:positive-regret} for the positive-noise systems.

\emph{Unique learned selector for Eq.~\eqref{ag:two-level}.}
Here $\tr\Sigma=D$ and $\cond(\Sigma)=c$.
For $L\in\mathcal S_{D,c}$ let $T=\tr(PL)$.
Then
\begin{equation}
 \frac{T}{r}\le\lambda_{\max}(L)
 \le c\lambda_{\min}(L)
 \le c\frac{D-T}{D-r},\qquad T\le crs.
 \label{ag:selector-bound}
\end{equation}
Since $\Sigma^{-1}=s^{-1}I_D-(c-1)(sc)^{-1}P$,
\[
 \tr(L\Sigma^{-1})=\frac{D}{s}-\frac{c-1}{sc}T\ge D,
\]
and $L=\Sigma$ attains equality.
Equality requires equality throughout Eq.~\eqref{ag:selector-bound}.
In particular, $\tr(P[\lambda_{\max}(L)I_D-L])=0$.
The bracket is positive semidefinite, so each vector in an orthonormal basis of $\operatorname{range}(P)$ has zero quadratic form under it and is annihilated by it.
Thus that range is a maximal-eigenvalue subspace of $L$.
The same argument makes its orthogonal complement a minimal-eigenvalue subspace.
Equality in the middle inequality and the trace constraint fix their eigenvalues to $sc$ and $s$, respectively.
Consequently $L=\Sigma$ is the unique trace minimizer.

Proposition~\ref{prop:learned_target_metric} gives $L\to\Sigma$ and $M_{T,\eta}\to qI_D$.
Lemma~\ref{ag:regret-convergence} now gives Eq.~\eqref{ag:compensation}, uniformly over all minimizer choices.
Lemma~\ref{ag:exact-rollouts} gives exact encoded terminal means and vanishing rollout MSE; the training losses vanish by Proposition~\ref{prop:learned_target_metric}.
Uniform convergence and $\Delta_H>0$ give strict improvement at small positive noise.
\end{proof}

The theorem concerns regret relative to the best feasible sequence, not zero terminal error: the selected goal is outside the limiting reachable ball.
It establishes a family in every dimension and at every fixed finite horizon.
The admissible noise threshold can depend on these parameters.
The learned-target conclusion uses Eq.~\eqref{ag:two-level} and its matching bound; the isotropic failure does not require the two-level restriction.
The theorem analyzes open-loop sequence selection over the stated feasible set.

The isotropic part of Theorem~\ref{ag:separation} proves Theorem~\ref{thm:finite_horizon_separation} in the main text; the additional two-level condition gives the learned-target compensation used in the method analysis.

\subsection{Stability of the planning separation}
\label{ag:stability}

The next lemma controls both changes to the planning cost and changes to the physical task cost.
It also allows approximate minimization of the perturbed latent cost.

\begin{lemma}[A quantitative regret margin]
\label{ag:margin}
Let $\mathcal U$ be a nonempty compact set, and let $b_0,j_0$ be continuous functions on it.
Put $j_0^*=\min_{\mathcal U}j_0$ and suppose
\[
 \Delta=\min_{U\in\operatorname*{arg\,min}_{\mathcal U} b_0}
 [j_0(U)-j_0^*]>0.
\]
For $q>0$, define $t_0=qj_0$ and
\begin{equation}
 \mathcal E=\{U\in\mathcal U:j_0(U)-j_0^*\le3\Delta/4\},\qquad
 \gamma=\min_{U\in\mathcal E}b_0(U)-\min_{U\in\mathcal U}b_0(U)>0.
 \label{ag:margin-gap}
\end{equation}
Let continuous costs $b,t,j$ satisfy
\[
 \norm{b-b_0}_\infty\le\epsilon_B,\qquad
 \norm{t-t_0}_\infty\le\epsilon_T,\qquad
 \norm{j-j_0}_\infty\le\epsilon_J.
\]
Let $U_B,U_T$ have respective optimization errors $\zeta_B,\zeta_T\ge0$, i.e., $b(U_B)\le\min_{\mathcal U}b+\zeta_B$ and $t(U_T)\le\min_{\mathcal U}t+\zeta_T$.
If $2\epsilon_B+\zeta_B<\gamma$, then
\begin{align}
 j(U_B)-\min_{\mathcal U}j&>3\Delta/4-2\epsilon_J,
 \label{ag:baseline-margin}\\
 j(U_T)-\min_{\mathcal U}j&\le(2\epsilon_T+\zeta_T)/q+2\epsilon_J.
 \label{ag:target-margin}
\end{align}
In particular, if also $\epsilon_J<\Delta/16$ and $2\epsilon_T+\zeta_T<q\Delta/8$, then the first regret exceeds $\Delta/2$ and the second is less than $\Delta/4$.
\end{lemma}

\begin{proof}
The set $\mathcal E$ is nonempty and compact, and is disjoint from $\operatorname*{arg\,min}_{\mathcal U}b_0$.
Thus $\gamma>0$.
The approximate minimizing inequality gives
\[
 b_0(U_B)\le\min_{\mathcal U}b_0+2\epsilon_B+\zeta_B
 <\min_{\mathcal U}b_0+\gamma,
\]
so $U_B\notin\mathcal E$.
Also $qj_0(U_T)\le qj_0^*+2\epsilon_T+\zeta_T$.
Finally $|\min_{\mathcal U}j-\min_{\mathcal U}j_0|\le\epsilon_J$; applying this bound and the pointwise task-cost bound proves Eqs.~\eqref{ag:baseline-margin}--\eqref{ag:target-margin}.
The stated thresholds give respectively a lower bound greater than $5\Delta/8$ and an upper bound less than $\Delta/4$.
\end{proof}

\subsection{Smooth nonlinear perturbations}
\label{ag:nonlinear-models}

We now realize the cost perturbations in Lemma~\ref{ag:margin} through nonlinear dynamics and encoders.
The reference models are the joint optima of Theorem~\ref{ag:separation}, including the two-level condition for its learned-target conclusion.
The nonlinear predictors may themselves be exact one-step conditional means.

\begin{corollary}[Nonlinear robustness of finite-horizon separation]
\label{ag:nonlinear}
Fix the parameters, action set, and goal of Theorem~\ref{ag:separation} with Eq.~\eqref{ag:two-level}.
For each sufficiently small $\eta>0$, choose any global linear-training optima with encoders $A_{B,\eta}$ and $A_{T,\eta}$.
Perturb the physical dynamics to
\begin{equation}
 X_{t+1}=F_\eta X_t+a_t+d(X_t,a_t)+\sqrt\eta Z_t,
 \quad X_0=0,\quad Z_t\overset{\mathrm{iid}}\sim\mathcal N(0,I_D),
 \label{ag:nonlinear-dynamics}
\end{equation}
where $d$ is continuously differentiable and $\sup_{x,\norm{a}\le\bar u}\norm{d(x,a)}\le\varepsilon_d$.
For $\chi\in\{\mathrm B,\mathrm T\}$, let
\begin{equation}
 \begin{aligned}
 f_{\chi}(x)&=A_{\chi,\eta}x+h_{\chi}(x),\qquad
 \sup_x\norm{h_{\chi}(x)}\le\varepsilon_e,\\
 \sup_x\norm{Dh_{\chi}(x)}&\le\varepsilon_e<\sigma_{\min}(A_{\chi,\eta}),
 \end{aligned}
 \label{ag:nonlinear-encoder}
\end{equation}
with $h_{\chi}$ continuously differentiable.
These bounds are global.
Define the exact encoded one-step conditional mean by
\begin{equation}
 p_{\chi}^{\mathrm{cm}}(z,a)
 =\mathbb E_Z f_{\chi}\!\left(F_\eta f_{\chi}^{-1}(z)+a+d(f_{\chi}^{-1}(z),a)+\sqrt\eta Z\right).
 \label{ag:nonlinear-mean}
\end{equation}
Use any continuous predictor $\widetilde p_{\chi}$ satisfying
\begin{equation}
 \sup_{z,\norm{a}\le\bar u}
 \norm{\widetilde p_{\chi}(z,a)-p_{\chi}^{\mathrm{cm}}(z,a)}\le\varepsilon_p.
 \label{ag:predictor-error}
\end{equation}
Roll it out from $f_{\chi}(0)$ and use the cost $c_{\chi}(U)=\norm{\hat z_{\chi,H}(U)-f_{\chi}(y)}^2$.
The physical cost is $j(U)=\mathbb E\norm{X_H(U)-y}^2$.
Suppose each returned sequence $\widetilde U_{\chi}$ has latent-cost optimization error at most $\zeta\ge0$.

There exist $\eta_0,\varepsilon_0,\zeta_0>0$, depending only on the fixed reference parameters, such that for $0<\eta<\eta_0$, $\max\{\varepsilon_d,\varepsilon_e,\varepsilon_p\}<\varepsilon_0$, and $0\le\zeta<\zeta_0$, every such choice satisfies
\begin{equation}
 \begin{aligned}
 \Regret_j(\widetilde U_B)&>\Delta_H/2,\qquad
 \Regret_j(\widetilde U_T)<\Delta_H/4,\\
 \Regret_j(U)&=j(U)-\min_{V\in\mathcal U}j(V).
 \end{aligned}
 \label{ag:nonlinear-regret}
\end{equation}
Both encoders are globally invertible and have uniformly nonsingular Jacobians.
For fixed $D,H$, their encoded rollout mean squared errors satisfy
\begin{equation}
 \sup_{U\in\mathcal U}\mathbb E\norm{\hat z_{\chi,H}(U)-f_{\chi}(X_H(U))}^2
 =O\!\left(\eta+(\varepsilon_d+\varepsilon_e+\varepsilon_p)^2\right).
 \label{ag:nonlinear-mse}
\end{equation}
The constants and conclusions are uniform over the reference global training optima.
In particular, Eq.~\eqref{ag:nonlinear-regret} allows $\varepsilon_p=\zeta=0$: exact one-step conditional-mean prediction and exact sequence optimization do not remove the separation in this nonlinear neighborhood.
\end{corollary}

\begin{proof}
For this proof write $A_{\chi}=A_{\chi,\eta}$, $F=F_\eta$, and $S_H(\beta)=\sum_{j=0}^{H-1}\beta^j$ for $\beta\ge0$.

\emph{1.
Invertibility and conditional means.}
By the derivative bound, $h_{\chi}$ is globally $\varepsilon_e$-Lipschitz.
For any $x,x'$,
\begin{equation}
 (\sigma_{\min}(A_{\chi})-\varepsilon_e)\norm{x-x'}
 \le\norm{f_{\chi}(x)-f_{\chi}(x')}
 \le(\norm{A_{\chi}}+\varepsilon_e)\norm{x-x'}.
 \label{ag:bilipschitz}
\end{equation}
For every $z$, the map $x\mapsto A_{\chi}^{-1}(z-h_{\chi}(x))$ is a contraction on $\mathbb R^D$, so it has a unique fixed point.
This proves surjectivity as well as injectivity.
The Jacobian is nonsingular by the same lower bound, so $f_{\chi}$ is a global $C^1$ diffeomorphism.
Thus Eq.~\eqref{ag:nonlinear-mean} is well defined and is the exact conditional mean of the next encoded state.
Its expectation is finite; bounded $h_{\chi}$ and continuity, together with the Gaussian first moment, also give continuity.

\emph{2.
Uniform perturbation of autoregressive prediction.}
Let $p_{\chi}^0(z,a)=A_{\chi}FA_{\chi}^{-1}z+A_{\chi}a$ and $\ell_{\chi}=\norm{A_{\chi}FA_{\chi}^{-1}}$.
At $z=f_{\chi}(x)$, zero-mean noise gives
\[
 \begin{aligned}
 p_{\chi}^{\mathrm{cm}}(f_{\chi}(x),a)-p_{\chi}^0(f_{\chi}(x),a)
 &=A_{\chi}d(x,a)-A_{\chi}FA_{\chi}^{-1}h_{\chi}(x)\\
 &\quad+\mathbb E_Z h_{\chi}(Fx+a+d(x,a)+\sqrt\eta Z).
 \end{aligned}
\]
Since $f_{\chi}$ is onto, for every $z$ and feasible $a$,
\begin{equation}
 \norm{\widetilde p_{\chi}(z,a)-p_{\chi}^0(z,a)}
 \le b_{\chi}:=\norm{A_{\chi}}\varepsilon_d+(1+\ell_{\chi})\varepsilon_e+\varepsilon_p.
 \label{ag:one-step-bound}
\end{equation}
Comparison with the reference rollout $A_{\chi}\mu_{\eta,t}(U)$ and induction yield
\begin{equation}
 \norm{\hat z_{\chi,H}(U)-A_{\chi}\mu_{\eta,H}(U)}
 \le\ell_{\chi}^H\varepsilon_e+S_H(\ell_{\chi})b_{\chi}.
 \label{ag:rollout-bound}
\end{equation}
Including the goal-encoding perturbation, put
\begin{equation}
 e_{\chi}=(\ell_{\chi}^H+1)\varepsilon_e+S_H(\ell_{\chi})b_{\chi},
 \qquad B_{\chi}=\norm{A_{\chi}}(r_H+\norm{y}).
 \label{ag:cost-bound-constants}
\end{equation}
The reference latent cost is $c_{\chi,\eta}^0(U)=\norm{A_{\chi}(\mu_{\eta,H}(U)-y)}^2$.
Using $\norm{F_\eta}\le\alpha$ and $\norm{\mu_{\eta,H}(U)}\le r_H$, we obtain
\begin{equation}
 \norm{c_{\chi}-c_{\chi,\eta}^0}_\infty\le2B_{\chi}e_{\chi}+e_{\chi}^2.
 \label{ag:latent-cost-bound}
\end{equation}
This argument does not commute expectation with a nonlinear rollout.

\emph{3.
Uniform perturbation of the physical cost.}
Couple Eq.~\eqref{ag:nonlinear-dynamics} to the unperturbed linear process $X_t^0$ using the same innovations and planned actions.
Pathwise,
\[
 \norm{X_{t+1}-X_{t+1}^0}\le\alpha\norm{X_t-X_t^0}+\varepsilon_d,
 \qquad\norm{X_H-X_H^0}\le d_H:=\varepsilon_dS_H(\alpha).
\]
Let $j_\eta^0(U)=\mathbb E\norm{X_H^0(U)-y}^2$ and define
\[
 B_*:=\left[(r_H+\norm{y})^2+\eta D S_H(\alpha^2)\right]^{1/2}.
\]
Cauchy--Schwarz and the pathwise bound give
\begin{equation}
 \norm{j-j_\eta^0}_\infty\le2B_*d_H+d_H^2.
 \label{ag:physical-cost-bound}
\end{equation}
The global bound on $d$ provides a uniform square-integrable envelope, so $j$ is continuous on $\mathcal U$ even though the innovations are unbounded.

\emph{4.
Applying the limiting regret margin.}
Define on the same compact action set
\[
 \begin{aligned}
 \mu_0(U)&=\sum_{t=0}^{H-1}\alpha^{H-1-t}a_t,
 &j_0(U)&=\norm{\mu_0(U)-y}^2,\\
 b_0(U)&=q\norm{\mu_0(U)-y}_{\Sigma^{-1}}^2,
 &t_0(U)&=qj_0(U).
 \end{aligned}
\]
By Theorem~\ref{ag:separation}, all minimizers of $b_0$ have $j_0$ regret $\Delta_H$, while $t_0$ has the same minimizers as $j_0$.
Propositions~\ref{prop:isotropic_joint_metric} and~\ref{prop:learned_target_metric}, together with $F_\eta\to\alpha I_D$, give, uniformly over reference training optima,
\[
 c_{B,\eta}^0\to b_0,\qquad c_{T,\eta}^0\to t_0,\qquad j_\eta^0\to j_0
\]
in the uniform norm on $\mathcal U$.
The encoder norms and inverse norms are uniformly bounded, so $\ell_{\chi}$, $B_{\chi}$, and all finite-horizon constants above are uniformly bounded as well.
Equations~\eqref{ag:latent-cost-bound} and~\eqref{ag:physical-cost-bound} therefore make the total cost perturbations arbitrarily small by choosing $\eta$ and $\varepsilon_d,\varepsilon_e,\varepsilon_p$ sufficiently small.
There is also a common positive lower bound on $\sigma_{\min}(A_{\chi})$, so Eqs.~\eqref{ag:nonlinear-encoder} and~\eqref{ag:bilipschitz} hold uniformly after reducing $\varepsilon_0$ if necessary.
Choose the thresholds to satisfy Lemma~\ref{ag:margin} with $\Delta=\Delta_H$, and then take a common sufficiently small $\zeta_0$.
This proves Eq.~\eqref{ag:nonlinear-regret} for every allowed choice, including nonunique or approximate planning minimizers.

\emph{5.
Encoded rollout error.}
Using the same coupling and the $L^2$ triangle inequality,
\[
 \begin{aligned}
 \left(\mathbb E\norm{\hat z_{\chi,H}-f_{\chi}(X_H)}^2\right)^{1/2}
 &\le\ell_{\chi}^H\varepsilon_e+S_H(\ell_{\chi})b_{\chi}+\norm{A_{\chi}}d_H+\varepsilon_e\\
 &\quad+\norm{A_{\chi}}\sqrt{\eta D S_H(\alpha^2)}.
 \end{aligned}
\]
All bounds are uniform in $U$ and the reference optima.
Squaring gives Eq.~\eqref{ag:nonlinear-mse}.
\end{proof}

Corollary~\ref{ag:nonlinear} concerns a neighborhood of the jointly selected linear models.
It proves stability of their planning behavior, not selection of nearby encoders by unrestricted nonlinear joint optimization.
For nonlinear encoders, an exact one-step encoded conditional mean need not compose into an exact terminal conditional mean; the explicit rollout bound above handles this distinction.
The global perturbation bounds accommodate the unbounded support of Gaussian innovations; compact action constraints alone do not bound noisy state trajectories.

\section{Experimental Details}
\label{app:experiments}

This appendix provides the implementation and evaluation details underlying the experiments in
Section~\ref{sec:experiments}. We first report the evaluation protocol and per-seed planning results,
then describe the latent-cost ordering diagnostic used in Table~\ref{tab:action_ranking}.
We next examine the target allocation learned by $\Lambda$Reg, sensitivity to the anisotropy bound,
and additional representation and prediction diagnostics.

\subsection{Experimental protocol and per-seed results}
\label{app:implementation}

\paragraph{Data and models.}
We use the TwoRoom, Reacher, PushT, and Cube datasets and preprocessing protocol of
LeWM~\citep{maes2026lewm}.
Training clips contain four observations with the temporal spacing used in that protocol.
A randomly initialized ViT-Tiny followed by a projection MLP produces
$192$-dimensional representations, and the causal predictor uses a three-observation context.
Prediction and planning operate in the original latent coordinates, whereas $\Lambda$Reg
is evaluated on the standardized representation $\Lambda^{-1/2}z$.

\begin{figure}[ht]
    \centering
    \includegraphics[width=\linewidth]{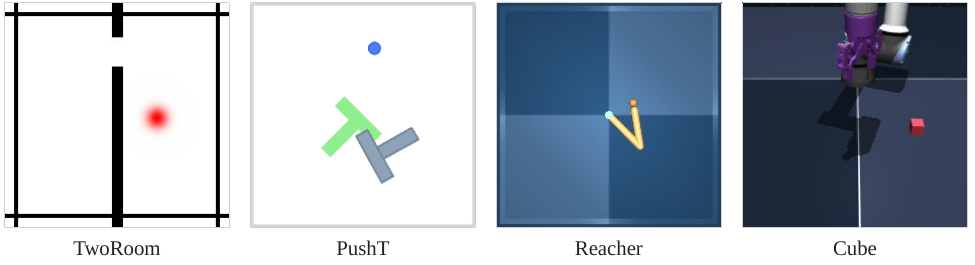}
    \caption{\textbf{Evaluation environments.}
    From left to right: TwoRoom, PushT, Reacher, and Cube.
    All four environments have continuous action spaces, and visual goal planning is performed
    directly from image observations.}
    \label{fig:environments}
\end{figure}

\paragraph{Training and planning settings.}
Table~\ref{tab:implementation_settings} summarizes the settings shared across the experiments.
Unless otherwise stated, AnisoWM uses $\kappa=2$ and $\lambda=0.09$.
The LeWM baseline corresponds to $\kappa=1$ and is trained with the same model architecture,
optimizer settings, training seeds, and evaluation pairs.

\begin{table}[ht]
    \centering
    \caption{\textbf{Implementation settings.}
    Target logits use zero weight decay and are clipped after each update as described in
    Section~\ref{sec:joint_training}.}
    \label{tab:implementation_settings}
    \small
    \begin{tabular}{@{}ll@{}}
        \toprule
        Setting & Value \\
        \midrule
        Epochs & 15 \\
        Training seeds & 3072, 1234, 42 \\
        Optimizer / batch size & AdamW / 128 \\
        Peak learning rate & $5\times10^{-5}$ \\
        Encoder and predictor weight decay & $10^{-3}$ \\
        Target-logit learning-rate schedule & Same as the model \\
        Representation dimension & 192 \\
        Primary anisotropy bound / regularization weight & $2$ / $0.09$ \\
        Projection directions / quadrature knots & 1,024 / 17 \\
        CEM iterations / candidates / elites & 30 / 300 / 30 \\
        Initial--goal pairs per environment & 50 \\
        Environment-step budget per pair & 50 \\
        \bottomrule
    \end{tabular}
\end{table}

\paragraph{Implementation pseudocode.}
Algorithm~\ref{alg:lambdareg} gives the implementation of $\Lambda$Reg.
Relative to SIGReg, the only change to the statistic is to standardize each latent coordinate by the learned diagonal Gaussian target before computing the projected empirical characteristic functions.
This standardization is used only inside the regularizer; prediction and planning operate on the original latent representation.

\begin{algorithm}[ht]
\caption{
\textbf{$\Lambda$Reg with the Epps--Pulley statistic and DDP support.}
The target variances are parameterized by logits $\alpha$ and normalized to have trace $D$.
After standardizing the features by the resulting diagonal target covariance, the remainder is the original SIGReg computation.
}
\label{alg:lambdareg}
\begin{lstlisting}[style=anisocode]
def LambdaReg(x, alpha, global_step, num_slices=1024):
    """x is (N, D); alpha is a learnable vector of size D."""

    # --- learned diagonal Gaussian target ---
    D = x.size(1)
    v = D * alpha.softmax(dim=0)   # sum(v) = D

    # standardize only inside LambdaReg
    x = x * v.rsqrt()

    # --- original SIGReg statistic ---
    # slice sampling -- synchronized across devices
    dev = dict(device=x.device)
    g = torch.Generator(**dev)
    g.manual_seed(global_step)

    proj_shape = (x.size(1), num_slices)
    A = torch.randn(proj_shape, generator=g, **dev)
    A /= A.norm(p=2, dim=0)

    # Epps--Pulley integration points and Gaussian window
    t = torch.linspace(0, 3, 17, **dev)
    exp_f = torch.exp(-0.5 * t**2)

    # empirical characteristic function -- gathered across devices
    x_t = (x @ A).unsqueeze(2) * t       # (N, M, T)
    ecf = (1j * x_t).exp().mean(0)
    ecf = all_reduce(ecf, op="AVG")

    # weighted squared distance to the Gaussian characteristic function
    err = (ecf - exp_f).abs().square().mul(exp_f)

    N = x.size(0) * world_size
    T = torch.trapz(err, t, dim=1) * N
    return T.mean()
\end{lstlisting}
\end{algorithm}

The target variances are
\[
v_i
=
D\frac{\exp(\alpha_i)}
{\sum_{j=1}^{D}\exp(\alpha_j)},
\]
which guarantees
\(
\operatorname{tr}(\Lambda)=D
\)
for
\(
\Lambda=\operatorname{diag}(v_1,\ldots,v_D)
\).
The logits are initialized at zero and, after each optimizer update, clipped as
\[
\alpha_i
\leftarrow
\operatorname{clip}
\left(
\alpha_i,
-\frac{1}{2}\log\kappa,
\frac{1}{2}\log\kappa
\right),
\]
which ensures
\(
\operatorname{cond}(\Lambda)\leq\kappa
\).
The target parameters receive gradients only through $\Lambda$Reg.

\paragraph{Evaluation protocol.}
Evaluation follows the protocol of LeWM~\citep{maes2026lewm}.
For each environment, planning is evaluated on 50 initial--goal pairs, with goals sampled
25 steps ahead within the same dataset trajectory.
CEM uses 300 candidate action sequences, 30 elites, and 30 optimization iterations, and
each evaluation has a 50-step environment budget.

\paragraph{Regularizer hyperparameters.}
We use the SIGReg settings reported by \citet{maes2026lewm}, including $\lambda=0.09$, 1,024 projection directions, and 17 integration knots, for both LeWM and AnisoWM. 
\citet{maes2026lewm} select $\lambda=0.09$ based on their reported SIGReg hyperparameter study, so we retain this value rather than retuning the baseline for AnisoWM.
$\Lambda$Reg additionally introduces the anisotropy bound $\kappa$; the primary comparison uses $\kappa=2$, and Appendix~\ref{app:sweep_details} reports its sensitivity.

\paragraph{Per-seed planning results.}
The main-paper planning results are means over three training seeds.
Table~\ref{tab:per_seed_results} reports the corresponding per-seed AnisoWM results.

\begin{table}[ht]
    \centering
    \caption{\textbf{Per-seed AnisoWM planning success (\%).}
    Each entry reports AnisoWM planning success under the same evaluation protocol
    used for Figure~\ref{fig:main_result}. The final column reports the mean across
    the three training seeds, which the main text rounds to an integer.}
    \label{tab:per_seed_results}
    \small
    \begin{tabular}{@{}lcccc@{}}
        \toprule
        Task & Seed 3072 & Seed 1234 & Seed 42 & Mean \\
        \midrule
        TwoRoom & 90 & 94 & 94 & 92.7 \\
        Reacher & 90 & 88 & 88 & 88.7 \\
        PushT   & 98 & 96 & 96 & 96.7 \\
        Cube    & 78 & 80 & 80 & 79.3 \\
        \bottomrule
    \end{tabular}
\end{table}

\subsection{Latent-cost ordering diagnostic}
\label{app:ranking_protocol}

The finite-horizon analysis in Section~\ref{sec:analysis} concerns the ordering of feasible
outcomes under the latent and task costs: a planner selects by that ordering, and the analysis
shows it can disagree with the task cost even when prediction is exact.
We therefore measure the ordering directly, in a diagnostic where LeWM and AnisoWM score the
same recorded action sequences and are compared against the outcomes those sequences reached.

\paragraph{Candidate sequences and outcomes.}
For each initial--goal pair, stored evaluations provide executed action sequences together with the
state each of them reached.
We label a sequence by the scalar of Table~\ref{tab:outcome_scalars}, evaluated at the state reached
after $H$ environment steps and normalized by the environment's success threshold, so that a value
of one is the success boundary and smaller values are better.
The label is continuous rather than binary, and is read at the horizon rather than minimized along
the trajectory.
The set excludes trajectories generated by either of the two checkpoints being compared, and both
models score the same remaining sequences.
A case is discarded when fewer than eight candidates remain, or when all candidates reach the same
outcome and the case induces no ordering.
The horizon $H$ is fixed per environment before the compared models are loaded, and is independent
of the CEM planning horizon.

\begin{table}[ht]
    \centering
    \caption{\textbf{Outcome scalar and diagnostic horizon.}
    Each scalar is normalized by its threshold, so one is the success boundary and lower is better.
    Cases are counted out of the fifty initial--goal pairs of the evaluation protocol.}
    \label{tab:outcome_scalars}
    \small
    \setlength{\tabcolsep}{6pt}
    \begin{tabular}{@{}llccr@{}}
        \toprule
        Task & Outcome scalar & Threshold & $H$ & Cases \\
        \midrule
        TwoRoom & agent--goal distance & $16$ px & 20 & 45 \\
        Reacher & largest joint-position error & $0.05$ rad & 25 & 50 \\
        PushT   & $\max\!\big(\text{position}/20\,\mathrm{px},\ \text{angle}/(\pi/9)\big)$ & $1$ & 25 & 48 \\
        Cube    & block--target distance & $0.04$ m & 15 & 33 \\
        \bottomrule
    \end{tabular}
\end{table}

\paragraph{Encoded and predicted latent costs.}
The cost a planner minimizes combines the learned representation with a predictor rollout, while
$\Lambda$Reg changes only the target the encoder is trained against.
To separate the two we score each sequence twice.
For a candidate sequence $U$ with realized terminal observation $o_H(U)$,
\begin{equation}
    J_{\mathrm{enc}}(U)
    =
    \|f_\theta(o_H(U))-f_\theta(o_g)\|^2,
    \qquad
    J_{\mathrm{pred}}(U)
    =
    \|\hat z_H(U)-f_\theta(o_g)\|^2.
    \label{eq:ranking_costs}
\end{equation}
$J_{\mathrm{enc}}$ applies the learned representation directly to the realized outcome and therefore
removes rollout prediction from the diagnostic.
$J_{\mathrm{pred}}$ is the cost available to the planner and additionally includes the predictor rollout.

\begin{figure}[ht]
    \centering
    \includegraphics[width=\linewidth]{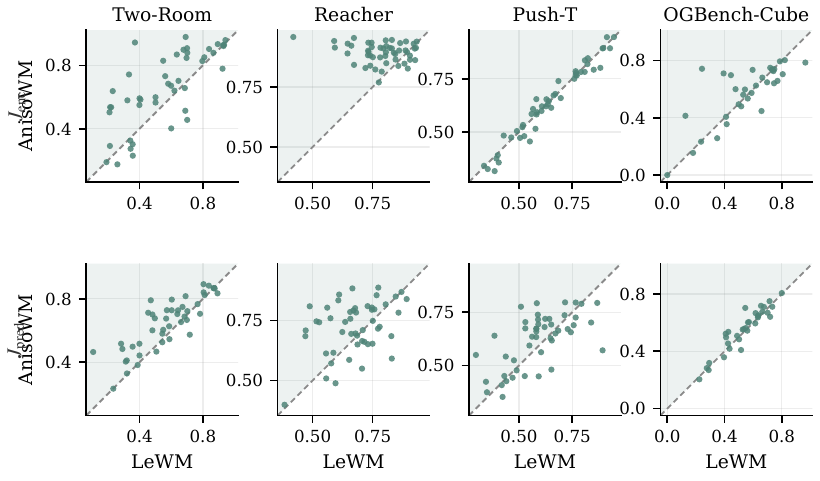}
    \caption{\textbf{Ranking accuracy, pair by pair.}
    Each point is one initial--goal pair: how well LeWM's cost orders its sequences on the
    horizontal axis against how well AnisoWM's does on the vertical one, both scoring the same
    sequences.
    The shaded half-plane above the diagonal is where AnisoWM orders that pair better.
    Top row: $J_{\mathrm{enc}}$, which leaves the representation on its own.
    Bottom row: $J_{\mathrm{pred}}$, the cost the planner minimizes.
    Rows of a column share their axis limits; each panel shows the seed
    Table~\ref{tab:action_ranking} reports.}
    \label{fig:action_ranking}
\end{figure}

\paragraph{Pairwise ordering accuracy.}
Within a case we consider every pair of candidates whose outcomes differ; pairs with identical
outcomes induce no ordering and are excluded.
The accuracy is the fraction of these pairs on which the latent-cost ordering agrees with the
outcome ordering, counting a pair with exactly equal costs as one half.
We evaluate this within each case and average over cases, so that each case contributes equally
regardless of its number of candidates.
Table~\ref{tab:action_ranking} reports this quantity for LeWM and AnisoWM on identical candidate
sets.

Comparing the two costs localizes where a difference between the models arises.
A gap under $J_{\mathrm{enc}}$ is present in the representation itself, before any rollout; a gap
that appears only under $J_{\mathrm{pred}}$ arises once the predictor is involved.
TwoRoom and Reacher show the former, with AnisoWM ahead under both costs, while in PushT the
representation-only ordering is nearly unchanged and only the predicted cost improves
(Figure~\ref{fig:action_ranking}).
\subsection{Additional local cost geometry}
\label{app:cost}

Figure~\ref{fig:cost-additional} shows four additional Cube/PushT example pairs using the visualization of Figure~\ref{fig:cost-main}.
The values above the zoomed panels report Spearman's rank correlation $\rho$ between latent planning cost and task cost within the evaluated region.
Across all additional examples, AnisoWM gives higher rank correlation than LeWM.

\begin{figure}[ht]
    \centering
    \includegraphics[width=0.82\linewidth]{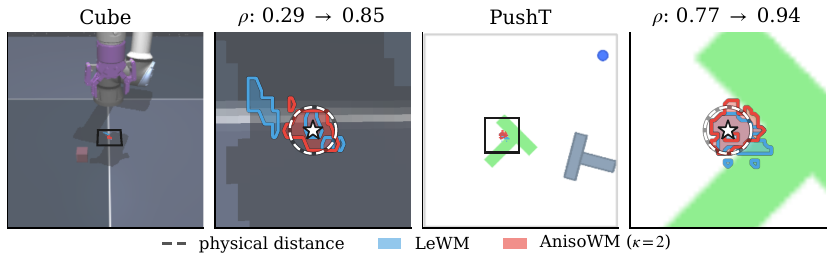}
    \par\vspace{0.25em}
    \includegraphics[width=0.82\linewidth]{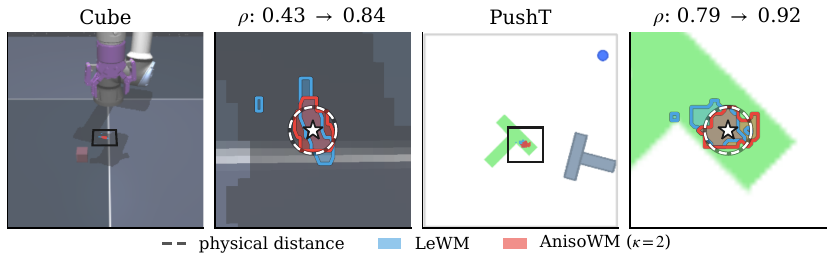}
    \par\vspace{0.25em}
    \includegraphics[width=0.82\linewidth]{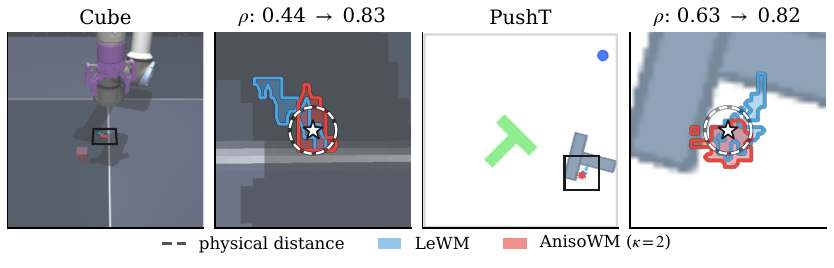}
    \par\vspace{0.25em}
    \includegraphics[width=0.82\linewidth]{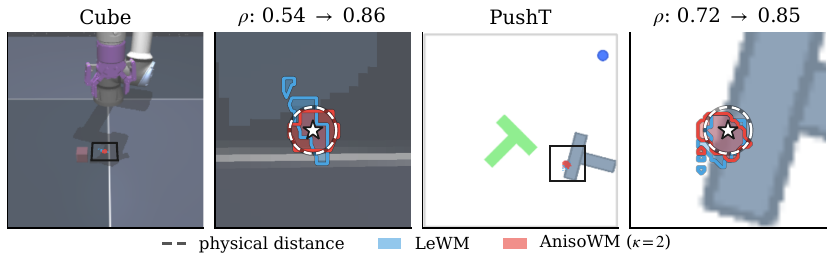}
    \caption{\textbf{Additional latent-cost neighborhoods around the goal.}
    Four additional Cube/PushT example pairs, shown using the same visualization as Figure~\ref{fig:cost-main}.
    The dashed circle denotes the task-cost neighborhood in physical space; blue and red contours denote the lowest-cost regions under LeWM and AnisoWM ($\kappa=2$), respectively.}
    \label{fig:cost-additional}
\end{figure}

\subsection{Learned target allocation}
\label{app:target_dynamics}

The trace and condition-number constraints determine the feasible family of target covariances but
do not specify how variance is allocated across latent coordinates.
This section examines how that allocation evolves during neural world-model training.

\paragraph{Environment-dependent target spectra.}
Figure~\ref{fig:target_spectra} in the main paper shows the learned target spectra at the common
bound $\kappa=2$.
The number of coordinates above the mean target variance differs across environments:
102 in TwoRoom, 124 in Reacher, 95 in PushT, and 97 in Cube.
Since the spectra are sorted independently, these numbers describe differences in the learned
variance distribution rather than aligned semantic coordinates.

\paragraph{Evolution during training.}
All target logits are initialized at zero, so training begins from the isotropic target.
The right panel of Figure~\ref{fig:target_spectra} shows the target spectrum throughout one
TwoRoom run.
The spectrum continues to evolve even after the learned condition number reaches the permitted
bound, indicating that reaching the boundary does not uniquely determine the variance allocation.

\paragraph{Use of the anisotropy budget.}
The learned target need not always saturate the condition-number constraint.
Figure~\ref{fig:target_condition} reports the target condition number during training for the
anisotropy bounds used in the TwoRoom sweep.
The bound is reached for $\kappa\leq8$, whereas the $\kappa=16$ run ends the fifteen-epoch
training budget at a condition number of $11.3$.
Thus the nominal bound need not equal the anisotropy realized by training.

\begin{figure}[ht]
    \centering
    \includegraphics[width=0.55\linewidth]{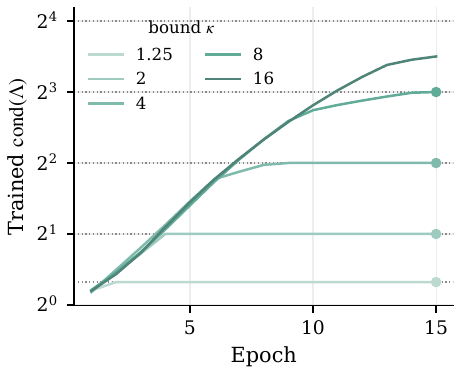}
    \caption{\textbf{Evolution of the learned target condition number.}
    Each curve shows one TwoRoom training run with the indicated anisotropy bound, and the dotted
    line at each bound is the ratio that bound permits.
    Filled markers indicate that the learned target reaches its bound.}
    \label{fig:target_condition}
\end{figure}

\subsection{Sensitivity to the anisotropy bound}
\label{app:sweep_details}

The primary experiments use the same $\kappa=2$ in all four environments.
To examine sensitivity to this choice, we additionally train models at several anisotropic bounds.
The numerical values underlying Figure~\ref{fig:kappa_sweep} are reported in
Table~\ref{tab:kappa_sweep}.

\begin{table}[ht]
    \centering
    \caption{\textbf{Planning success across anisotropic target bounds (\%).}
    Each entry corresponds to the training run used in the sweep of
    Figure~\ref{fig:kappa_sweep}.}
    \label{tab:kappa_sweep}
    \small
    \begin{tabular}{@{}ccccc@{}}
        \toprule
        $\kappa$ & TwoRoom & Reacher & PushT & Cube \\
        \midrule
        1.25 & 90 & 90 & 94 & 82 \\
        2    & 90 & 90 & 98 & 78 \\
        4    & 92 & 84 & 92 & 78 \\
        8    & 92 & 84 & 92 & 78 \\
        16   & 94 & 86 & 92 & 78 \\
        \bottomrule
    \end{tabular}
\end{table}

The preferred value differs across environments.
In particular, increasing the permitted variance contrast is not monotonically beneficial:
TwoRoom continues to improve across the displayed range, whereas the other environments attain
their best anisotropic result at smaller bounds.
These sweep results are single training runs and should therefore be interpreted as sensitivity
diagnostics rather than as replacements for the three-seed primary comparison.

\subsection{Realized representation geometry}
\label{app:training_diagnostics}

The target covariance $\Lambda$ is used inside $\Lambda$Reg, but it does not directly constrain
the empirical covariance of the learned neural features.
We therefore measure the centered feature covariance on a fixed validation probe.

\paragraph{Feature covariance spectrum.}
For each run, let
\[
    \ell_1 \ge \ell_2 \ge \cdots \ge \ell_D
\]
be the eigenvalues of the measured feature covariance.
To compare spectra across runs, we normalize each eigenvalue by the mean eigenvalue.

\begin{figure}[ht]
    \centering
    \includegraphics[width=0.65\linewidth]{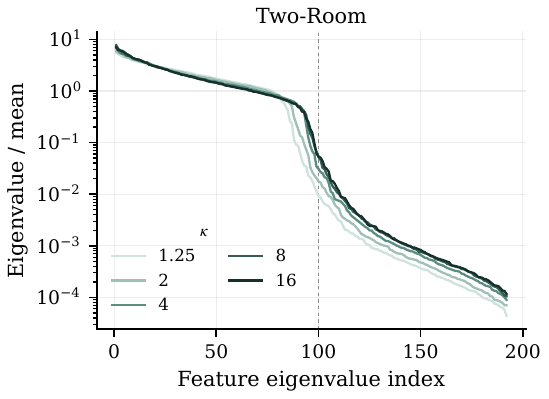}
    \caption{\textbf{Measured feature-covariance spectra in TwoRoom.}
    Eigenvalues of the centered feature covariance are normalized by their mean and sorted
    independently for each displayed anisotropy bound; the dashed line marks index $100$, the
    value quoted in the text.
    These empirical feature covariances are distinct from the learned target covariance $\Lambda$.}
    \label{fig:feature_spectrum}
\end{figure}

Changing $\kappa$ changes the measured feature spectrum as well as the target spectrum.
For example, in TwoRoom the eigenvalue at index $100$, normalized by the mean eigenvalue,
increases from $0.009$ at $\kappa=1.25$ to $0.054$ at $\kappa=16$.

\paragraph{Representation variation.}
With
\[
    \pi_j=\frac{\ell_j}{\sum_k\ell_k},
\]
we compute the entropy effective rank
\begin{equation}
    r_{\mathrm{ent}}
    =
    \exp\!\left(
        -\sum_{j:\pi_j>0}\pi_j\log\pi_j
    \right).
\end{equation}
Across the reported runs, the entropy effective rank ranges from $74$ to $111$ out of
$192$ dimensions.
This empirical statistic is distinct from the participation-ratio lower bound on the target covariance
in Corollary~\ref{cor:boundaries}.

\subsection{Additional prediction diagnostics}
\label{app:prediction_diagnostics}

\paragraph{Latent prediction loss and planning.}
Figure~\ref{fig:prediction_planning} compares latent prediction loss with planning success across
the anisotropy sweep.
The run with the smallest latent MSE is not consistently the run with the highest planning success.
Because each model learns its own representation, however, latent MSE is measured in a different
coordinate system for each run and should not be interpreted as a common physical prediction error.
The figure is therefore included only as a diagnostic of the relationship between the reported latent
training loss and planning performance.

\begin{figure}[ht]
    \centering
    \includegraphics[width=\linewidth]{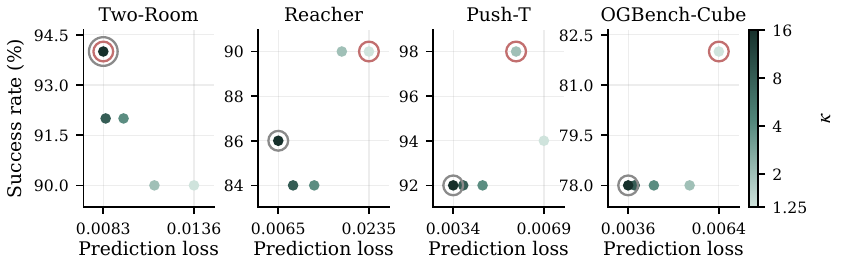}
    \caption{\textbf{Latent prediction loss and planning success.}
    One marker per trained run, coloured by its anisotropy bound.
    A run further left predicts better and a run higher plans better.
    The red ring marks the best-planning run and the grey ring the lowest-loss one; where a single
    run is both, the two rings are concentric.
    The MSE is measured in each model's own learned representation space.}
    \label{fig:prediction_planning}
\end{figure}

\subsection{Qualitative rollouts}
\label{app:qualitative_rollouts}

Figures~\ref{fig:rollout_examples_a} and~\ref{fig:rollout_examples_b} show recorded rollout
pairs for the four environments.
These examples illustrate behavior under the same evaluation planner and are not used as an
aggregate measure of performance.

\begin{figure}[p]
    \centering
    \includegraphics[
        width=\linewidth,
        height=0.78\textheight,
        keepaspectratio
    ]{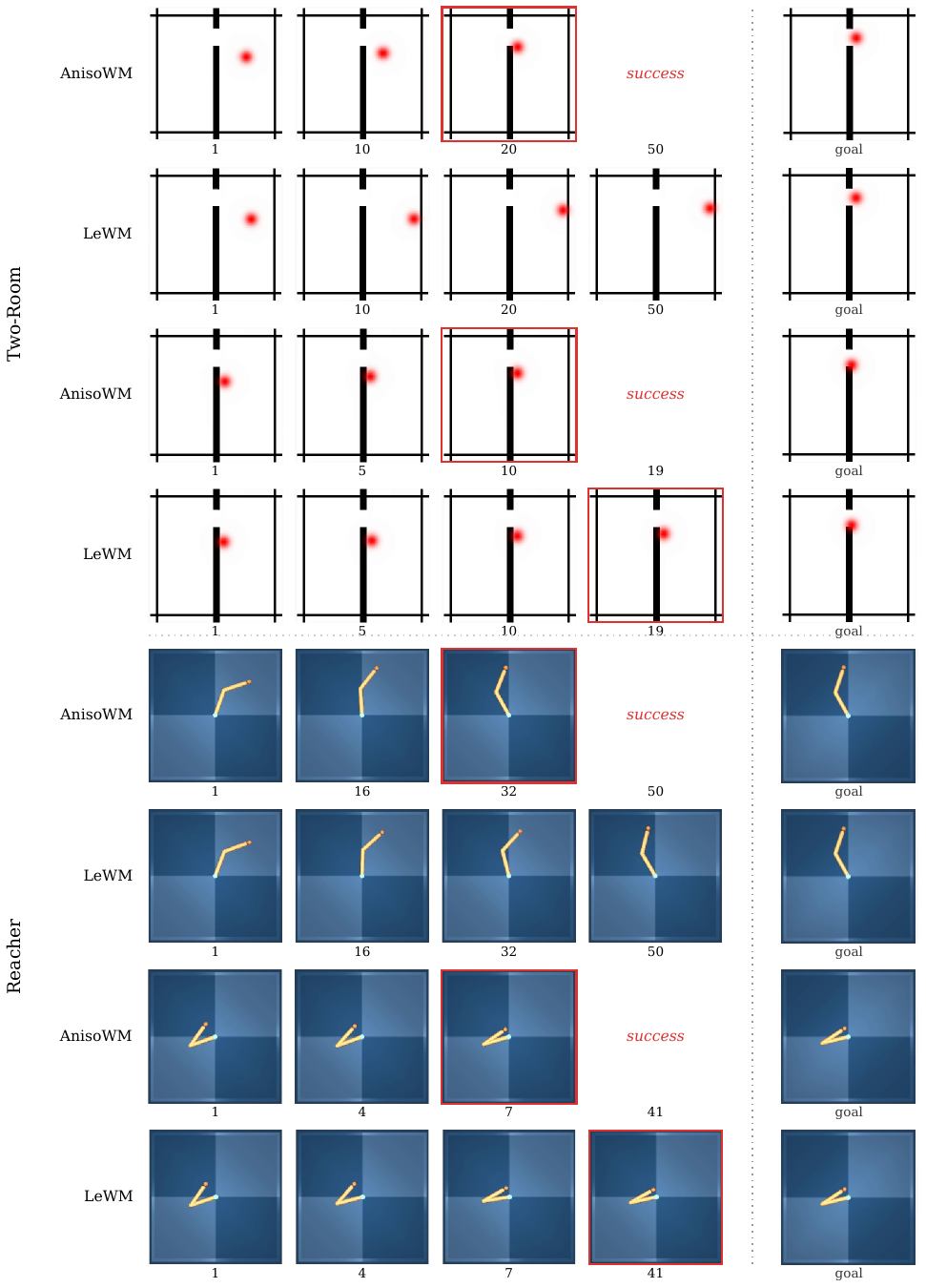}
    \caption{\textbf{Rollouts in TwoRoom and Reacher.}
    Two rollout pairs per environment compare AnisoWM at $\kappa=2$ with LeWM under the same
    evaluation planner and seed.
    Columns show the indicated environment steps and the goal observation; red outlines indicate
    recorded goal attainment.}
    \label{fig:rollout_examples_a}
\end{figure}

\begin{figure}[p]
    \centering
    \includegraphics[
        width=\linewidth,
        height=0.78\textheight,
        keepaspectratio
    ]{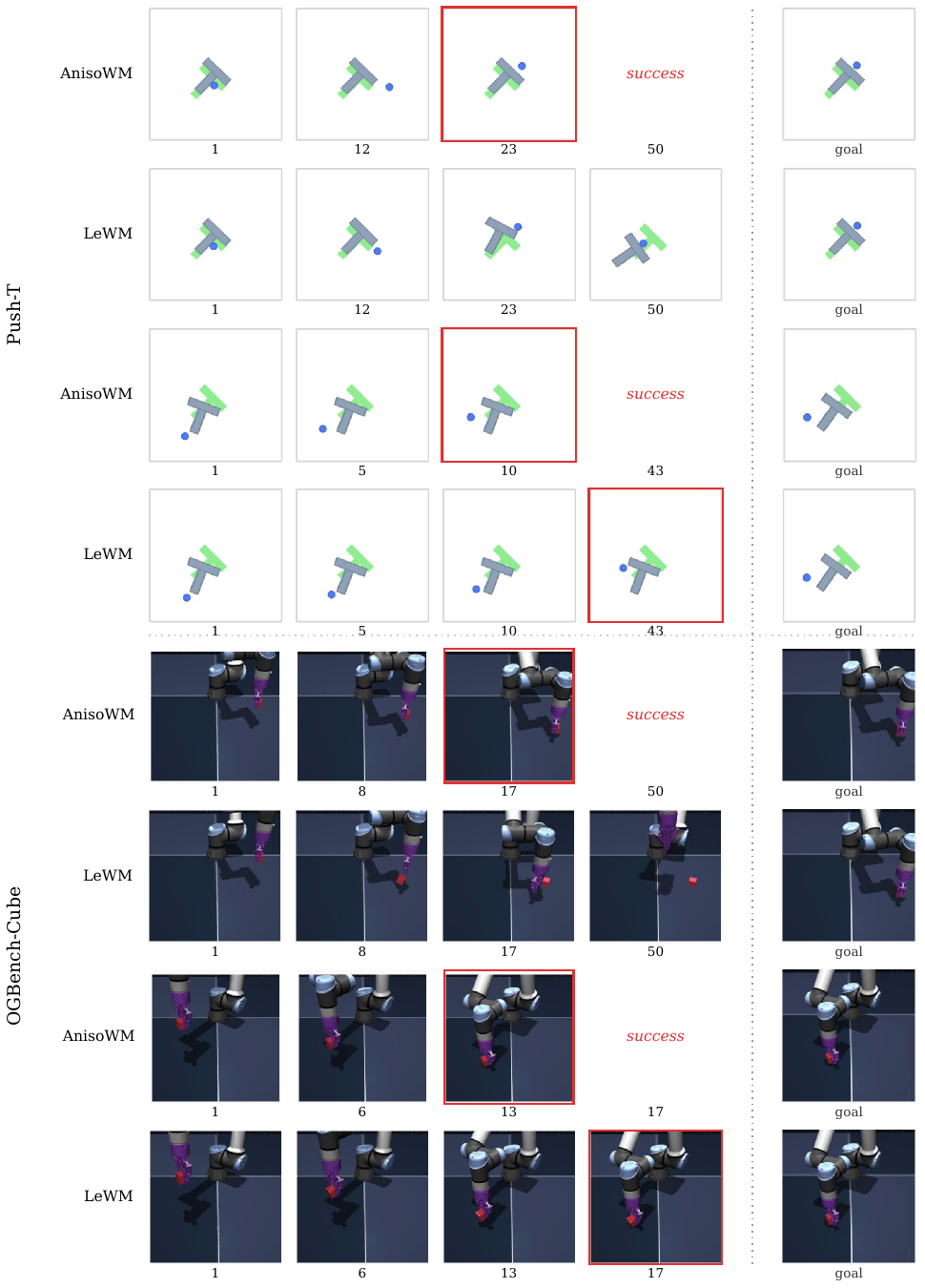}
    \caption{\textbf{Rollouts in PushT and Cube.}
    Two rollout pairs per environment compare AnisoWM at $\kappa=2$ with LeWM under the same
    evaluation planner and seed.
    Columns show the indicated environment steps and the goal observation; red outlines indicate
    recorded goal attainment.}
    \label{fig:rollout_examples_b}
\end{figure}

\end{document}